\documentclass{article}

\usepackage{microtype}
\usepackage{graphicx}
\usepackage{booktabs} 

\usepackage{hyperref}

\usepackage[accepted]{icml2024}

\usepackage{amsmath}
\usepackage{amssymb}
\usepackage{mathtools}
\usepackage{amsthm}

\usepackage[capitalize,noabbrev]{cleveref}

\usepackage[utf8]{inputenc} 
\usepackage[T1]{fontenc}    
\usepackage{hyperref}       
\usepackage{url}            
\usepackage{bm}
\usepackage{amsfonts}       
\usepackage{nicefrac}       
\usepackage{xcolor}         

\usepackage{svg}
\usepackage{caption}
\usepackage{subcaption}

\usepackage[most]{tcolorbox}

\def\eqref#1{equation~\ref{#1}}

\def\1{\bm{1}}

\def\vtheta{{\bm{\theta}}}
\def\va{{\bm{a}}}

\def\vg{{\bm{g}}}
\def\vh{{\bm{h}}}

\def\vp{{\bm{p}}}

\def\vs{{\bm{s}}}
\def\vt{{\bm{t}}}

\def\vv{{\bm{v}}}

\def\vx{{\bm{x}}}
\def\vy{{\bm{y}}}

\def\mA{{\bm{A}}}

\def\mC{{\bm{C}}}

\def\mF{{\bm{F}}}

\def\mH{{\bm{H}}}
\def\mI{{\bm{I}}}
\def\mJ{{\bm{J}}}

\def\mL{{\bm{L}}}

\def\mP{{\bm{P}}}

\def\mR{{\bm{R}}}
\def\mS{{\bm{S}}}

\def\mV{{\bm{V}}}
\def\mW{{\bm{W}}}

\DeclareMathAlphabet{\mathsfit}{\encodingdefault}{\sfdefault}{m}{sl}
\SetMathAlphabet{\mathsfit}{bold}{\encodingdefault}{\sfdefault}{bx}{n}

\newcommand{\softmax}{\mathrm{softmax}}

\DeclareMathOperator*{\argmax}{arg\,max}

\DeclareMathOperator{\Tr}{Tr}

\newcommand{\where}{\textbf{where }}

\newcommand{\when}{\textbf{when }}
\newcommand{\When}{\textbf{When }}
\newcommand{\what}{\textbf{what }}

\newcommand{\mybox}[2]{
\begin{tcolorbox}[colback=blue!5!white, colframe=blue!75!black]
\paragraph{#1}
#2
\end{tcolorbox}
}

\newcommand{\expect}[1]{\mathbb{E} \left [ #1 \right ]}
\newcommand{\inv}[1]{{#1}^{-1}}

\theoremstyle{plain}
\newtheorem{theorem}{Theorem}[section]

\newtheorem{lemma}[theorem]{Lemma}
\newtheorem{corollary}[theorem]{Corollary}
\theoremstyle{definition}

\theoremstyle{remark}

\icmltitlerunning{Self-Expanding Neural Networks}

\begin{document}

\twocolumn[
\vspace*{-1em}
\icmltitle{Self-Expanding Neural Networks}
\icmlsetsymbol{equal}{*}

\author{%
  Rupert Mitchell$^1$ \quad Martin Mundt$^{1,2}$ \quad Kristian Kersting$^{1,2,3,4}$\\
  $^1$Department of Computer Science, TU Darmstadt, Darmstadt, Germany \\
  $^2$Hessian Center for AI (hessian.AI), Darmstadt, Germany \\
  $^3$German Research Center for Artificial Intelligence (DFKI), Darmstadt, Germany \\
  $^4$Centre for Cognitive Science, TU Darmstadt, Darmstadt, Germany \\
  \texttt{\{rupert.mitchell,martin.mundt,kersting\}@cs.tu-darmstadt.de} \\
}

\begin{icmlauthorlist}
\icmlauthor{Rupert Mitchell}{1,2}
\icmlauthor{Robin Menzenbach}{1}
\icmlauthor{Kristian Kersting}{1,2,3,4}
\icmlauthor{Martin Mundt}{1,2}
\icmlauthor{\normalfont $^1$Department of Computer Science, TU Darmstadt, Darmstadt, Germany}{} 
\icmlauthor{\normalfont $^2$Hessian Center for AI (hessian.AI), Darmstadt, Germany}{}
\icmlauthor{\normalfont $^3$German Research Center for Artificial Intelligence (DFKI), Darmstadt, Germany}{}
\icmlauthor{\normalfont $^4$Centre for Cognitive Science, TU Darmstadt, Darmstadt, Germany}{}
\end{icmlauthorlist}

\icmlcorrespondingauthor{Rupert Mitchell}{rupert.mitchell@tu-darmstadt.de}

\icmlkeywords{Neural Networks, Dynamic Representational Capacity}

\vskip 0.3in
]




\begin{abstract}
The results of
training a neural network are heavily dependent on the architecture chosen;
and even a modification of only its size,
however small,
typically involves restarting
the training process.
In contrast to this,
we begin training with a small architecture,
only increase its capacity as necessary
for the problem,
and avoid
interfering with
previous optimization
while doing so.
We thereby introduce
a natural gradient based approach
which intuitively expands both the width and depth of a neural network
when this is likely to substantially reduce the hypothetical converged training loss.
We prove an upper bound on the ``rate'' at which neurons are added,
and a computationally cheap lower bound on the expansion score.
We illustrate the benefits of such Self-Expanding Neural Networks with full connectivity and convolutions in both classification and regression problems, 
including those where the appropriate architecture size is substantially uncertain a priori.
\end{abstract}

\section{Introduction}
Correctly tailoring a model's capacity to an arbitrary task is extremely challenging, especially when the latter is not yet well studied.
This challenge can be side stepped by choosing an architecture which is so large that a poor solution is nevertheless unlikely to occur \cite{DBLP:conf/iclr/NakkiranKBYBS20}, e.g. due to the double-descent phenomenon. However, since it is hard to predict what size would be large enough this will often in practice entail using a massively overparameterized neural network (NN)  \cite{DBLP:conf/cvpr/SzegedyLJSRAEVR15,DBLP:journals/cacm/KrizhevskySH17,DBLP:conf/nips/HuangCBFCCLNLWC19}.
Surely it is possible to detect that the existing capacity of the network is insufficient and add more neurons when and where they are needed?
In fact, biological NNs are grown by adding new neurons to the existing network through neurogenesis. 
The popular review \cite{Gross2000} discusses the relatively recent discovery that this process is still active in the adult mammalian brain
\cite{Vadodaria2014},
and \cite{Kudithipudi2022,Draelos2017} identify it as a key ability underpinning lifelong learning.
Thus inspired,
we propose an analogous process for adding both neurons and layers to an artificial NN during training, 
based on a local notion of ``sufficient capacity'' derived from first principles in close relation to the natural gradient \cite{DBLP:journals/neco/Amari98,DBLP:journals/jmlr/Martens20}. 

Although conceptually intuitive, any method for artificial neurogenesis
must answer three challenging questions to avoid the problem of locally insufficient capacity \cite{DBLP:conf/iclr/EvciMUPV22}.
It must determine \textbf{when} the current capacity is insufficient and that neuron(s) must therefore be added.
It must identify \textbf{where} these neurons should be introduced.
Finally, it must choose \textbf{what} initialization is appropriate for these neurons.
These questions, if they are addressed at all in the literature, are normally addressed piecemeal or in ad-hoc ways. For example, very few methods address the question of \textbf{what} \cite{DBLP:conf/iclr/EvciMUPV22,DBLP:conf/nips/WuLS020}.
\When is answered either by assuming predetermined schedules \cite{DBLP:conf/nips/WuLS020,DBLP:journals/corr/RusuRDSKKPH16},
or by waiting for the training loss to converge \cite{DBLP:conf/iclr/YoonYLH18} \cite{DBLP:conf/nips/WuW019},
neither of which are informative about \textbf{where}.
``Whenever you parry, hit, spring, ..., you must cut the enemy in the same movement.''\footnote{
Miyamoto Musashi, The Book of Five Rings (circa 1645)
}
Our metaphorical enemy is not a loss which is momentarily poor, or even one which is converging to a poor value:
it is a deficiency in our parameterization such that the optimizer cannot make progress.
We argue that by inspecting the degrees of freedom of the optimizer in function space,
one may not only strike faster in answer to \textbf{when}, but answer \textbf{where} and \textbf{what} in the same stroke.

\begin{table*}[t]
\vspace*{-0.5em}
\caption{Existing expansion methods' (lack of) answer to when, where, what and whether they consider depth. SENN is the only approach to provide a cohesive answer to all three questions based on natural gradients.}
\label{growing_table}
\begin{center}
\resizebox{.95\textwidth}{!}{%
\begin{tabular}{lllll}
\bf METHOD & \bf WHEN & \bf WHERE & \bf WHAT & \bf DEPTH?\\
\hline
Dynamic Node Creation \cite{doi:10.1080/09540098908915647} & converged loss & preset & random & No\\
Progressive NNs \cite{DBLP:journals/corr/RusuRDSKKPH16} & at new task & preset & random & No\\
Neurogenesis DL \cite{Draelos2017} & recon error & recon error & random & No\\
Dynamically Exp. NNs \cite{DBLP:conf/iclr/YoonYLH18} & converged loss & preset then prune & random & No\\
\hline
Net2Net \cite{net2net2016} & manual & everywhere & fixed & Yes \\
ActiveNAS \cite{DBLP:conf/nips/GeifmanE19} & end of training & argmax on presets & reinitialize & Yes\\
\hline
Splitting Steepest Descent \cite{DBLP:conf/nips/WuW019} & converged loss & loss reduction & loss reduction & No \\
Firefly Architecture Descent \cite{DBLP:conf/nips/WuLS020} & N epochs & vanilla gradient & loss reduction & No\\
GradMax \cite{DBLP:conf/iclr/EvciMUPV22} & future work & future work & vanilla gradient & No \\
\hline
Self-Expanding NNs (SENN, ours)         &natural gradient &natural gradient &natural gradient & Yes
\end{tabular}
}
\end{center}
\label{tab:growing_table}
\vspace*{-2.5em}
\end{table*}

From a mathematical perspective, these degrees of freedom available to the optimizer
are given by the image of the parameter space under the Jacobian,
and the derivative with respect to the loss in function space will not in general lie in this subspace.
It is however possible to project this derivative onto that subspace,
and the natural gradient, $\mF^{-1} \vg$,
is exactly the change in parameters which changes the function according to this projection.
In order to measure the size of that projection for a given parameterization,
we introduce the natural expansion score $\eta = \vg^T \mF^{-1} \vg$ in this work.
Specifically, the capacity of a neural network is locally insufficient when this score is small for the current parameterization. In turn, this allows us to formulate Self-Expanding Neural Networks, a new family of dynamic models where we 
add neurons \textbf{when} this substantially increases $\eta$, \textbf{where} they will maximally increase $\eta$, and choose \textbf{what} initialization to use for the new parameters according to how it increases $\eta$. To summarize, our \textbf{contributions} are:
\begin{enumerate}
    \item We introduce the \textit{natural expansion score} which measures the increase in rate of loss reduction under natural gradient descent when width or depth is added to a NN. 
    \item We show how such additions may be made during training
    without altering the function represented by the network.
    Our neurogenesis inspired Self-Expanding Neural Networks (SENN) thus avoid interfering with previous optimization or requiring restarts of training.
    \item We prove that SENN's number of simultaneously added neurons is bounded and introduce a computationally efficient provable lower bound to increases in natural expansion score resulting from additions.
    \item We demonstrate SENN's effectiveness for regression and classification, for fully-connected and convolutional  variants, and highlight how SENN yields stable architecture states through continuous expansion and even neuron pruning when perturbed on purpose.
\end{enumerate}

\section{Related Methods for Growing Neural Nets}
\label{sec:related_work}

The problem of adding nodes to NNs during training has been contemplated for over 30 years (e.g. Dynamic Node Creation \cite{doi:10.1080/09540098908915647}),
but remains substantially unsolved.
There does not seem to exist a unified answer to \textbf{when}, \textbf{where}, and \textbf{what},
as summarized in table \ref{tab:growing_table}.
Most methods cannot add depth and sideline at least one of these questions.

Inspired by neurogenesis like SENN, \citet{Draelos2017} examine the case of representational learning with stacked autoencoders,
where they exploit local reconstruction error to determine \when and \where to add neurons.
Due to their more general setting,
Dynamic Node Creation,
Progressive NNs \cite{DBLP:journals/corr/RusuRDSKKPH16} and Dynamically Expandable NNs \cite{DBLP:conf/iclr/YoonYLH18}
use simple training loss convergence or even task boundaries to answer \when, but must then fall back on ad-hoc preset decisions for \textbf{where}. 
All four methods freeze old neurons or continue training from their present values,
but randomly initialize new neurons in answer to \what.
While ActiveNAS \cite{DBLP:conf/nips/GeifmanE19} can add both width \emph{and} depth,
it does so by completely restarting training with a fresh initialization of the whole network after every modification. Similarly, Net2Net \cite{net2net2016} can be interpreted as a proof of concept that allows insertion of layers; however, \textbf{when} to add is decided manually and additions happen \textbf{everywhere} with fixed initialization.

\begin{figure*}[t!]
\begin{subfigure}{0.425\textwidth}
    \includegraphics[width=\textwidth]{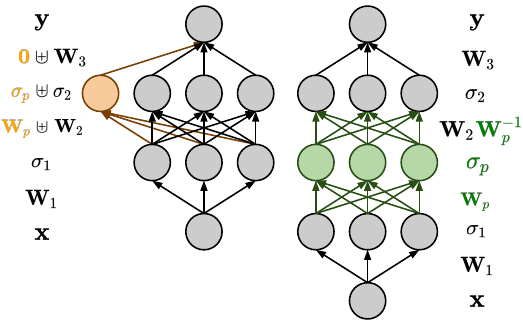}
    \caption{In the case of full connectivity (MLP), the model's component functions are composed vertically, e.g. $\mathbf{\sigma}_1(\mW_1 \vx)$ gives the first set of hidden activations.
    $\uplus$ indicates concatenation along a hidden dimension, i.e. width addition.
    $\mW_2 \mW_p^{-1}$ indicates matrix multiplication,
    used in depth expansion via the insertion of an identity function. \label{fig:addition_diagram_MLP}}
\end{subfigure}
\hfill
\begin{subfigure}{0.545\textwidth}
\includegraphics[width=\textwidth]{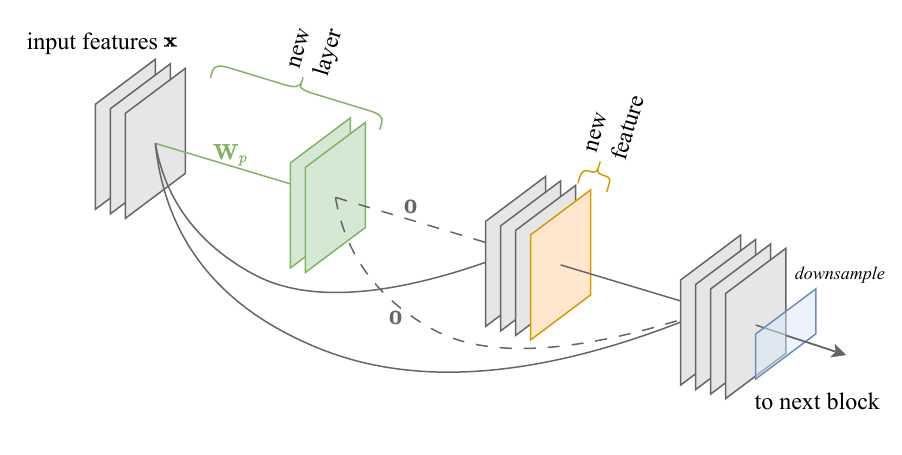}
\caption{In convolutional architectures, width addition follows the same principle as in the fully-conected MLP scenario (figure panel a), i.e. concatenation of hidden activations along the hidden dimension (with additional spatial dimensions). In contrast, insertion of depth no longer requires earlier constraints on activation functions or weight invertability, if skip connections are leveraged for the purpose of retaining the initial identity mapping. \label{fig:addition_diagram_CNN}}
\end{subfigure}
\caption{SENN can add width (orange) and depth (green) to a neural network without changing the overall function.}
\label{fig:addition_diagram}
\end{figure*}

The final cluster of three methods all aim to improve on random initialization as an answer to \textbf{what}.
Splitting Steepest Descent \cite{DBLP:conf/nips/WuW019} and Firefly \cite{DBLP:conf/nips/WuLS020} make small changes to the existing function and answer \what by optimizing the consequent loss reduction.
The former answers \when by waiting for convergence and examining the loss, whereas the latter simply adds more capacity every $N$ epochs.
GradMax \cite{DBLP:conf/iclr/EvciMUPV22} is the closest to SENN in spirit,
but is based on vanilla rather than natural gradient descent. More importantly, potential extensions of the method to the \where and \when questions are mentioned briefly and their investigation deferred to future work.
All three of these latter methods are only able to avoid redundancy of added neurons with existing neurons to the extent that the network is already converged. Of these three, only GradMax completely avoids changing the overall function.
Although a follow-up~\cite{maile2022} to Gradmax has made progress on remaining questions,
it comes with the trade-off of considering activations in place of gradients (discounting the objective),
with depth addition unaddressed.

In contrast, SENN provides a \textbf{monolithic answer} to \textbf{all three questions} via the \emph{natural expansion score} for both width and depth.

\newcommand{\useful}{\textit{useful}}
\newcommand{\leverage}{\textit{leverage}}
\newcommand{\range}{\textit{range}}
\newcommand{\redundant}{\textit{redundant}}
\newcommand{\ie}{\textit{i.e.}}
\newcommand{\eg}{\textit{e.g.}}

\section{Self-Expanding Neural Networks}
Consider introducing a new parameter $\theta$ to our network, initialized to zero, and ignore for now the presence of the other parameters.
Clearly, if we want to reduce the loss by changing our new parameter, the value of this parameter must affect the loss.
In particular, we will say our parameter has high \leverage{} if the derivative of the loss with respect to our parameter is large, \ie{} it receives a large gradient $g$.
The reduction in loss we can achieve with this new parameter is proportional to both the magnitude of this gradient and the distance to the local optimum with this parameter.
This highlights a second criterion we must consider: we will say our parameter has a large \range{} if the distance to the local optimum $\theta^* - \theta$ is large, \ie{} we can move this parameter a long way before it stops reducing the loss.
If we presume any interactions with other parameters to be accounted for, 
our ability to reduce the loss by changing some new parameter is thus captured by the product of its \leverage{} and \range.

\subsection{How to add: expanding the model architecture without changing the overall function}
\label{sec:how_to_add}
In order that learning is not disrupted it is important that when we add capacity to our network we do not thereby change the overall function it represents.
The simplest class of modifications satisfying this condition leave existing parameters unchanged and merely introduce new ones.
It is thus not necessary for new parameters to compensate for the loss of old ones, only for the new parameters to initially have no effect.
In all architectures, width addition via concatenation in the hidden dimension can be arranged to satisfy these conditions if the output weights of the new neuron are initialized to zero, since existing parameters may be left unchanged. It is similarly important to avoid changing the overall function when inserting a layer into a neural network,
but this may sometimes require modifying or replacing existing parameters rather than simple concatenation.
The core challenge is that existing information flow must be routed around the new layer unchanged, either via skip connections or some other identity map. 

Figure \ref{fig:addition_diagram_MLP} shows the case of width addition for an MLP in orange, where the linear transform $\mW_2$ of the second layer is expanded by concatenation to $\mW_p \uplus \mW_2$.
So long as the subsequent transform is expanded with zeros to form $\textbf{0} \uplus \mW_3$, $\mW_p$ may be arbitrary and leave the function unchanged.
While $\mW_p$ will thereby receive zero gradients, the newly introduced zero weights will not and will therefore not remain zero for long.
Layer insertion for MLPs is shown in green:
a linear transform $\mW_2$ is replaced by $(\mW_2 \inv{\mW}_q) \circ (\sigma_q = \mI) \circ \mW_q$ where $\mW_q$ must be invertible, but is otherwise arbitrary.
In this case we further require that our nonlinearity $\sigma_q$ is parameterized and for some choice of parameters can represent the identity at initialization.

Figure \ref{fig:addition_diagram_CNN} shows width and depth addition for a DenseNet CNN block in orange and green respectively.
As before, width is added via concatenation in the feature dimension and the initialization of output weights to zero.
Notably, and conveniently for CNNs in particular, we neither require invertibility of the new parameters, nor do we require the nonlinearity to be parameterized, if skip connections are present. Layer insertion is then more straightforward and follows similar principles to width addition: old direct connections become skip connections, and output weights of the new layer are initialized to zero with input weights arbitrary.
We thus have the first ingredient of SENN:

\mybox{\textcolor{blue!75!black}{SENN Ingredient 1:} How to add capacity without changing the overall function.}{
We add proposed neurons $p$ to layer $i$
by concatenation along the $i$th hidden dimension
$(\mathbf{0} \uplus \mW_{i+1}) \circ (\sigma_p \uplus \sigma_i) \circ (\mW_p \uplus \mW_i) = \mW_{i+1} \circ \sigma_i \circ \mW_i$,
and initialize the output weights of $p$ to zero.
In CNNs we exploit skip connections to generalize the same method to layer insertion.
For MLPs, we insert a new layer $q$ by replacing some linear transform $\mW_i$
with the composition $(\mW_i \inv{\mW}_q) \circ (\sigma_q = \mI) \circ \mW_q$,
where $\mW_q$ is invertible and $\sigma_q$ is initialized to the identity.
}

When working with MLPs, we must therefore choose a suitable parameterized activation function.
Rational activation functions satisfy our conditions and were shown to obtain good real world performance \cite{DBLP:conf/iclr/0001SK20}.
We use the simplified form
$\sigma_\vtheta (x) = \alpha x + (\beta + \gamma x)/(1+x^2)$,
where $\vtheta = \{ \alpha, \beta, \gamma \}$ are the three parameters of $\sigma$,
and setting $\vtheta = \{ 1, 0, 0 \}$ results in the identity function, as required.
Since this parameter count is small, we do not share the activation function weights within our layers.

\subsection{What about the interactions with other parameters?}
We will now turn to the question of the interactions between new and existing parameters.
When we say that a new parameter is \redundant{} with the existing parameters, we mean that whatever change to the network predictions, and therefore loss, we could make having introduced the new parameter is already possible with some combination of the existing parameters.
There is also a possible positive interaction we have not yet mentioned: perhaps there is some beneficial combination of changes to the new and existing parameters which would not be beneficial if made separately.
If we could find the joint local optimum $\vtheta^*$ of all these parameters together, we would automatically account for both of these effects.
Fortunately, if we approximate the local loss function in terms of its first and second derivatives, there is a closed form solution to this problem:
\begin{equation}
    \vtheta^* = \vtheta - \mH^{-1} \vg
\end{equation}
where $\vg$ and $\mH$ are the first and second derivatives (gradient and Hessian) of the loss function respectively.
Recalling that we earlier expected the possible reduction in loss to be proportional to both the \leverage{} or gradient, $\vg$, and the \range{} or distance to the local optimum, $\mH^{-1}\vg$,
we can also examine the corresponding closed form solution for that reduction.
Specifically, we consider the second order truncation of the loss function taylor series, and introduce $\eta$ as some approximation thereof:
\begin{equation}
    \frac{1}{2}\eta(\vtheta) :\approx
    \frac{1}{2} \vg_\vtheta^T \mH_\vtheta^{-1} \vg_\vtheta =
    \mathcal{L}(\vtheta) - \mathcal{L}(\vtheta_*)
    + \mathcal{O}\left( |\vtheta_* - \vtheta|^3 \right)
    \label{eqn:eta_as_loss_second_order_truncation}
\end{equation}
where we have introduced $\vtheta$ subscripts to represent the dependence of the derivatives on where they are evaluated.

The value of $\eta(\vtheta)$ is a measure of how much we can expect to reduce the current loss using some particular set $\vtheta$ of parameters, given a quadratic approximation to the local loss surface.
It accounts for the \leverage{} of the parameters $\vtheta$, their \range, and the extent to which they are mutually \redundant.
To get a sense for how much better the combination $\vtheta_p \uplus \vtheta_0$ of some new proposed parameters $\vtheta_p$ with the existing parameters $\vtheta_0$ is than the existing parameters alone, we need only examine the corresponding difference $\Delta \eta (\vtheta_p) = \eta(\vtheta_p \uplus \vtheta_0) - \eta(\vtheta_0)$.
In particular:
\mybox{\textcolor{blue!75!black}{SENN Ingredient 2:} What initialization to use for new parameters.}{
When we must choose \what{} initialization to use for (some subset of) our new parameters $\vtheta_p$, we make this choice in order to maximize the resulting increase $\Delta \eta(\vtheta_p)$.
}

It is now also possible to specify \where{} to add parameters.
\mybox{\textcolor{blue!75!black}{SENN Ingredient 3:} Where to add capacity.}{
Assuming the mode of addition (width or depth) to already be chosen, we add parameters $\vtheta_p$ \where{} they will maximize the resulting increase $\Delta \eta(\vtheta_p)$.
}
When choosing between proposals we maximize $\eta$ as above, but the space of initializations is large and we must choose a finite set to compare.
In practice we could directly optimize for $\Delta \eta$ or use some sampling procedure, for example parameters in proposals may be drawn from the standard initializing distributions for neural networks.

\subsection{When to add new capacity}
We have seen that when choosing between possible ways to add new capacity it is sufficient to maximize the increase $\Delta \eta$ in expansion score,
but the question remains of when to add capacity.
We answer this question by requiring the associated increase to be ``sufficient'':
we require the relative increase $\Delta \eta / \eta_0$ over the current score $\eta_0$ to surpass an \textbf{expansion threshold} $\tau$.
If the best proposal under consideration surpasses this threshold we add it and we repeat until such a proposal no longer exists.
If $\eta_0$ is very small due to the network approaching convergence this might result in proposals with negligible effects on the final loss being accepted,
and so we further require the increase $\Delta \eta$ to also surpass an absolute \textbf{stopping criterion} $\alpha$.

\mybox{\textcolor{blue!75!black}{SENN Ingredient 4:} When to add capacity.}{
We add capacity \when{} the relative increase $\Delta \eta/ \eta_0$ exceeds $\tau$
and the absolute increase exceeds $\alpha$.
}

While we answer \textbf{when}, \where and \what cohesively with $\eta$ during training,
the stopping criterion can therefore be seen as analogous to monitoring for loss plateaus,
and we thus concur with all prior works on terminating training.

\subsection{SENN's mathematical and practical perspective}
\label{sec:peudoinverse}
We have seen that dynamically sized neural networks may be constructed by reference to a second order expansion of the loss function (equation \ref{eqn:eta_as_loss_second_order_truncation}), but the exact hessian $\mH$ is not convenient to work with.
In practice, we approximate it with the fisher matrix $\mF$, which, unlike $\mH$, is always positive definite and whose further properties we will now explore.
Considering the whole dataset of $N$ examples simultaneously, our neural network can be seen as a function $\Theta \rightarrow \mathcal{Y}$ from parameter space to the concatenation of all outputs it produces for every data item.
All directions in the space of such functions expressible by our parameters $\vtheta \in \Theta$ occur as images $\mJ \vt \in \mathcal{T}(\mathcal{Y})$ under the jacobian $\mJ$ of the neural network at $\vtheta$ of some tangent vector $\vt \in \mathcal{T}_\vtheta(\Theta)$ to the parameter space.
By abuse of notation we call the space of such directions $\mJ (\Theta)$.
By differentiating the loss function, we may obtain a gradient with respect to the concatenated outputs $\vg_y$.
The various choices of definition for the Fisher matrix $\mF$ correspond to choices of metric on $\mathcal{Y}$,
but for the purposes of simplicity we choose the euclidean metric, corresponding to $\mF := \frac{1}{N}\mJ^T\mJ$.
The results of this section generalize to other metrics by change of basis.
The factor of $\frac{1}{N}$ exists to counterbalance the mean over examples in the standard gradient definition $\vg := \frac{1}{N} \mJ^T \vg_y$.

The direction taken in function space by natural gradient descent is the projection $\mP_\Theta$ of $\vg_y$ onto $\mJ (\Theta)$, according to our implicitly chosen metric.
This can be seen via the properties of the Moore-Penrose pseudoinverse $\mJ^+$ and the substitution $\mJ \mF^{-1} \mJ^T \vg_y = \mJ \mJ^+ \vg_y = \mP_\Theta(\vg_y)$.
By further substituting the definition of $\vg = \frac{1}{N} \mJ^T \vg_y$ we obtain
\begin{equation}
    \eta = \vg^T \mF^{-1} \vg = \frac{1}{N}||\mP_\Theta(\vg_y)||_2^2
\end{equation}
where we can see that the natural expansion score $\eta$ corresponds to the squared length of the component of the output gradient expressible by our current parameters, normalized by the dataset size $N$.
This has the immediate consequence that $\eta$ is bounded above by $\lambda := \frac{1}{N} ||\vg_y||_2^2$, as the projection cannot be longer than the original.
When we expand our parameter set $\Theta$, we increase the dimension of $\mJ (\Theta)$ and, in general, also the squared length of the projection.
This corresponds exactly to the increase in $\Delta \eta$ which we consider when deciding whether to add capacity. \\

\textbf{Bounded addition.}
Having shown $\eta$ to be bounded above by $\lambda$,
we will now sketch a proof for a corresponding bound on the number of simultaneous additions $N_s$.
From the relative expansion threshold $\tau$, we have that $\eta_{i+1} > (1 + \tau) \eta_i$ after each successive addition.
This means that $\eta$ approaches its bound exponentially with successive additions, yielding $N_s < (\ln \lambda - \ln \eta_0)/\ln (1 + \tau)$.
Since $\mF^{-1}$ is positive semi-definite, we have $\eta_0 \geq 0$,
and by using the stopping criterion $\alpha$ we have $\eta_1 - \eta_0 > \alpha$.
This gives us a lower bound $\eta_1 > \alpha$ whose substitution into our first bound on $N_s$ removes the dependency on $\eta_0$.\\

\begin{theorem}[\textbf{Upper bound on the ``rate'' of neuron addition}]
The maximum number of additions $N_s$ from repeatedly running the expansion algorithm is bounded:
$N_s < 1 + (\ln \lambda - \ln \alpha)/\ln (1 + \tau)$.
\end{theorem}
(Full proof in appendix.)
For example, if $\tau = 1$ and $\alpha / \lambda > 10^{-3}$ then $N_s < 1 + \frac{3\ln10}{\ln2} < 11$.

\textbf{Choice of curvature approximation.}
Since it is computationally impractical to work with the full fisher $\mF$, we must approximate.
In the extreme case one might choose the identity matrix, at which point $\eta$ becomes the squared magnitude of the gradient.
However, consider adding a second copy of an existing parameter which receives a large gradient.
This parameter will receive the same gradient, and unboundedly many more copies may be introduced, each contributing the same increase in squared magnitude.
This failure mode stems from a failure to account for correlation between neurons,
and diagonal approximations of $\mF$ will have this same problem.
We thus consider the KFAC~\cite{DBLP:conf/icml/MartensG15} approximation $\mF \approx \Tilde{\mF} = \mS \otimes \mA$ for each layer separately,
where the two kronecker factors are the input activation second moment $\mA$ of a layer and a pre-activation gradient second moment $\mS$.
The $\mA$ factor accounts for redundancy between incoming signals, whereas $\mS$ accounts for redundancy between neurons in subsequent layers.

\begin{figure*}[t!]
\begin{center}
\includegraphics[height=0.3533\textwidth]{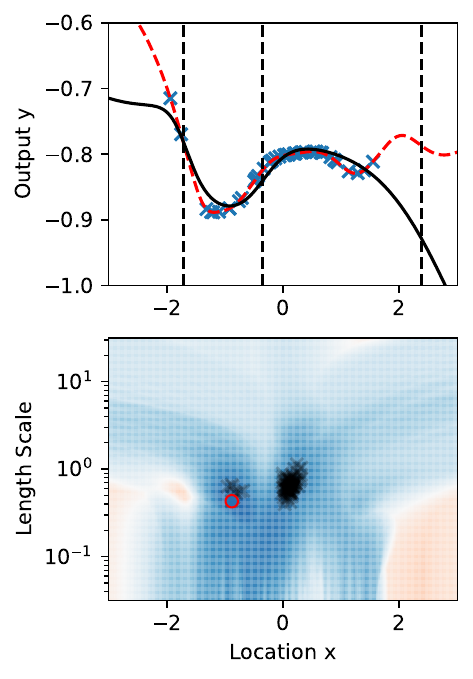}
\includegraphics[height=0.35\textwidth]{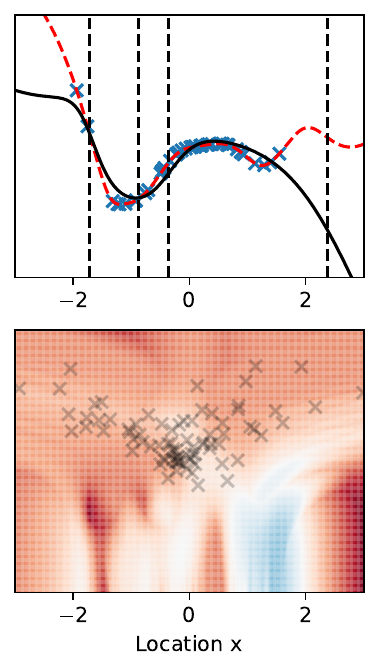}
\includegraphics[height=0.35\textwidth]{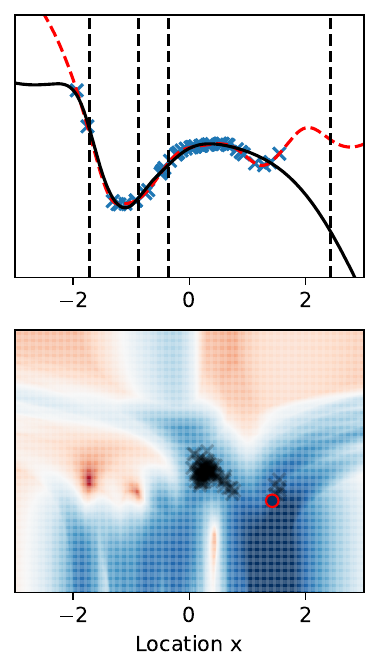}
\includegraphics[height=0.35\textwidth]{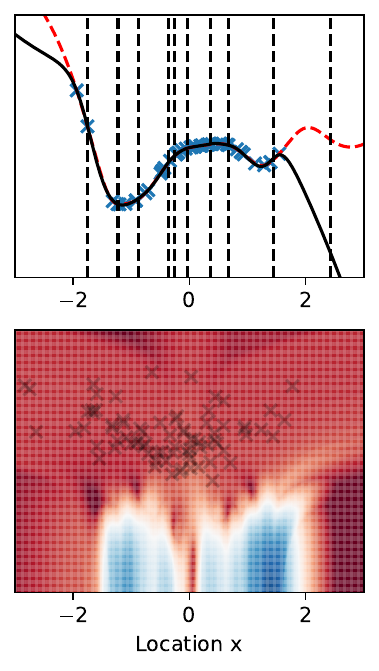}
\includegraphics[height=0.35\textwidth]{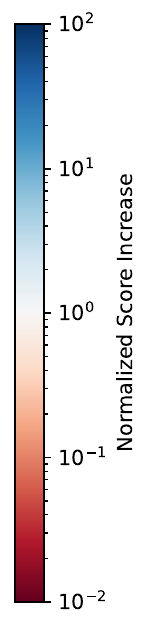}
\end{center}
\vskip -0.1in
\caption{
A single layer SENN (black, solid) is trained on a target function (red, dashed)
via least-squares regression on samples (blue, markers).
Vertical lines show the location of existing neurons.
The lower panels show $\Delta \eta' / \eta_0$ as a function of the location and scale of the nonlinearity introduced by a new neuron.
Accepted and rejected proposals are marked in red and black respectively.
From left to right we see the landscape before and immediately after the fourth neuron is added,
before the fifth is added, and at the end of training.
SENN adds neurons where relevant in order to achieve a good fit.
}
\label{fig:regression}
\end{figure*}

This approximation is advantageous due to the inverse formula $\Tilde{\mF}^{-1} = \mS^{-1} \otimes \mA^{-1}$.
The corresponding contribution to the expansion score from a layer $l$ may be written as the trace $\eta = \Tr[\mS_l^{-1}\partial \mW \mA_{l-1}^{-1} \partial \mW^T]$ where the gradient for the linear weights $\mW_l$ is given by the correlation $\partial \mW = \expect{\vg_l \va_{l-1}^T}$ between input activations from the preceding layer $\va_{l-1}$ and output gradients $\vg_l$.
Let the residual gradient $\vg_r = \vg - \expect{\vg_l \va_{l-1}^T} \mA^{-1} \va$ be that part of the output gradients not predicted by the existing activations.
Then if $\va_p$ is the activation vector of a set of proposed neurons in layer $l-1$,
and $\mA_p$ is their second moment:
\begin{theorem}[\textbf{Computationally cheap lower bound on increase in natural expansion score within a layer \normalfont{\textit{l}}}]
$\Delta \eta' := \Tr[\mA_p^{-1} \expect{\va_p \vg_r^T} \mS_l^{-1} \expect{\vg_r \va_p^T}]$
is a lower bound $\Delta \eta' \leq \Delta \eta = \eta_p - \eta_0$ on the improvement in natural expansion score due to a proposed addition of neurons $p$ to~$l$. 
\end{theorem}
(proof in appendix via block LDU decomposition of joint activation covariance).
The reader will note that if $\mA_p$ is small, \eg{} adding a single neuron, then many such proposals may be considered cheaply, as no large inversions beyond $\Tilde{\mF}^{-1}$ are required.
In the rarer case of layer addition, further approximations may be made if necessary, \eg{} considering the mean $\Delta \eta'$ for each new neuron considered separately.
Further tricks for working with $\Tilde{\mF}^{-1}$ such as rank one updates, stochastic estimation, and accounting for curvature due to activation functions are also detailed in the appendix.

\subsection{Pruning is reversed addition}
\label{sec:pruning}
At any time we are considering adding new features we have access to the KFAC representation of $\mF$.
Not only does this allow the easy approximation of the cost of removing an existing feature \cite{optimal_brain_damage},
but it allows both the adjustment of existing neurons to compensate for this removal,
and the inclusion of this adjustment in the estimate of removal costs \cite{optimal_brain_surgeon}.
We combine this estimated removal cost with the associated reduction in expansion score,
and are thereby able to both expand and prune in a cohesive manner.
This allows us to compensate for potential excessive expansion with later pruning,
and, novelly, for excessive pruning through later expansion.


\section{Experiments}
\label{sec:experiments}

We now illustrate the behavior of the natural expansion score through an example of least-squares regression and provide empirical intuition for depth insertion through visualization of decision boundaries in binary classification. We then demonstrate SENN's efficacy in dynamic architecture selection in popular image classification. Finally, we demonstrate architecture stability through alternating expansion and pruning in the presence of perturbations and conclude with SENN's use in transfer learning of a pre-trained model.
The implementation code we use in these experiments is available at \url{https://github.com/ml-research/self-expanding-neural-networks}.

\subsection{Width addition in least-squares regression}
We start by showing that the evolution over training of the possible improvements $\Delta \eta'$ in natural expansion score due to potential width expansions is meaningful. To this end, consider the application of a single layer SENN to a one dimensional least squares regression task  as shown in figure \ref{fig:regression},
i.e. SENN with depth addition deliberately disabled.
With basis functions given by the neurons of that layer, 
we can plot the normalized score increase $\Delta \eta' / \eta_c$ of the best neuron for each basis function location and length scale.
Where $\Delta \eta' / \eta_c > 1$ there exists an acceptable proposal.
Accepted/rejected proposed neurons are shown on this landscape in red/black at key points in training.
We see in the leftmost panel that the best such proposal is accepted because it achieves a large improvement in $\eta$,
and it corresponds to a basis function location close to datapoints with visibly large prediction error which we have been unable to reduce using the existing neurons.
The next panel shows the same landscape after the new neuron is introduced,
and it can be seen that the $\Delta \eta' / \eta_c$ values for neurons with similar locations to it have been dramatically reduced
since they would be redundant.
The second panel from the right shows the result of optimizing the new expanded parameters until the point at which the next neuron would be added.
It can be seen that the prediction errors in the region of the previously introduced neuron are now practically invisible,
and that the next neuron is to be introduced in a different region in which errors remain.
The final panel shows the function approximation at the conclusion of training. The prediction errors are now negligible and proposals with large relative increase in $\eta$ are not to be found in the region considered. \\
Overall, SENN thus identifies regions of locally insufficient capacity in our parameterization and
targets these regions precisely with new added neurons to achieve a good fit.

\begin{figure}[t!]
\begin{center}
    \includegraphics[width=0.9\columnwidth]{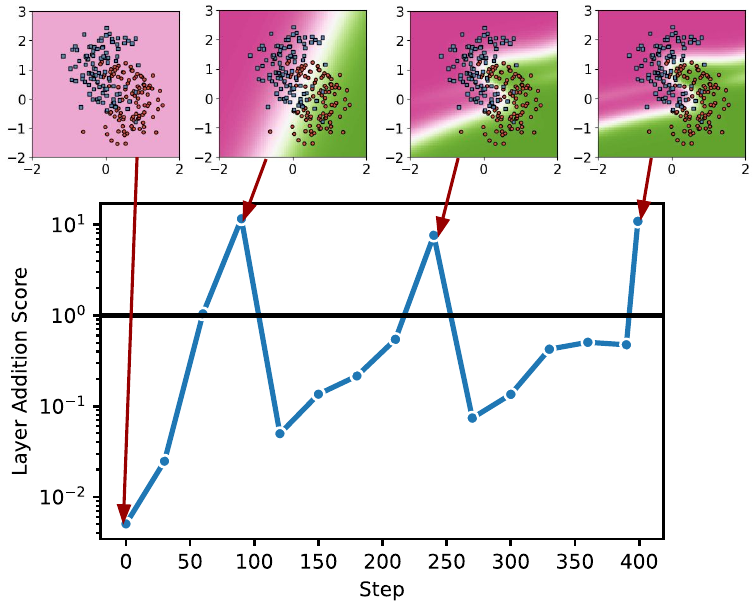}
\end{center}
    \vskip -0.1in
    \caption{
    2-D binary classification with SENN. The normalized layer addition score $\Delta \eta' / \eta_c$ is shown as a function of optimization steps; the horizontal bar shows the point above which a layer will be added.
    The score increases during three phases during which the SENN has initial zero, one and then two hidden layers.
    These layer insertions allow SENN to represent more complex decision boundaries (shown at the top)
    when required for global expressivity.
    \label{fig:classification}}
\vspace*{-0.5em}
\end{figure}

\subsection{Depth (layer) addition in 2-D binary classification}
Let us now transition to the insertion of depth.
In figure \ref{fig:classification} we plot $\Delta \eta' / \eta_c$ for the best layer addition proposals as a function of overall optimizer steps for the two-dimensional inputs from the half-moons dataset \cite{scikit-learn}.
Visualizations of the learned decision boundary at initialization and just before layer additions are shown.
We can observe that $\Delta \eta' / \eta_c$ increases approximately monotonically during three phases,
punctuated by large drops when layers are added.
In the initial phase the network has zero hidden layers (i.e. is linear),
and the simplicity of the decision boundary at the end of this phase reflects this.
Since the datapoints are not linearly separable,
the large $\Delta \eta' / \eta_c$ value correctly indicates that the introduction of a hidden layer is necessary in order to further reduce loss.
The visible increase in decision boundary complexity and accuracy over the course of the second phase confirms this.
The beginning of the third phase marks the introduction of a second hidden layer and we wait until $\Delta \eta' / \eta_c$ rises again,
indicating an exhaustion of this new capacity, before reexamining the decision boundary.
The increase in boundary complexity is less visible this time, but close inspection reveals that the boundary has become narrower and more rounded.
In summary, we have intentionally constructed a scenario where depth addition is necessary
and have seen that SENN inserts new layers when this is necessary for global expressivity.

\begin{figure}[t!]
\begin{center}
\begin{subfigure}{0.32\columnwidth}
\centering
\includegraphics[height=2.5cm]{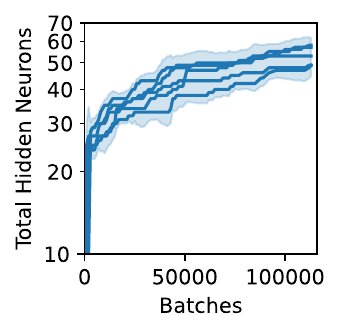}
\end{subfigure}
\begin{subfigure}{0.32\columnwidth}
\centering
\includegraphics[height=2.5cm]{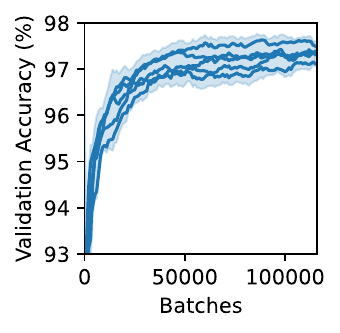}
\end{subfigure}
\begin{subfigure}{0.32\columnwidth}
\centering
\includegraphics[height=2.5cm]{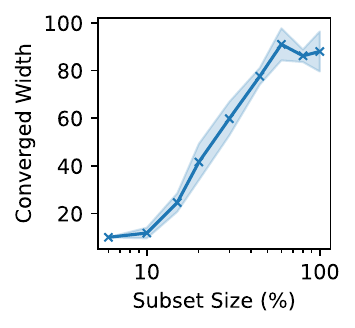}
\end{subfigure}
\end{center}
\vskip -0.15in
\caption{SENN shows reasonable and reproducible hidden layer growth on MNIST at appealing any-time validation accuracy without intermittent perturbations (left pair of panels). 
SENN features appropriate scaling with respect to data complexity in its chosen network sizes (right panel).}
\label{fig:mnist}
\vspace*{-0.8em}
\end{figure}

\subsection{Dynamic architecture size in image classification}
\begin{figure*}
    \centering
    \includegraphics[width=0.48 \textwidth]{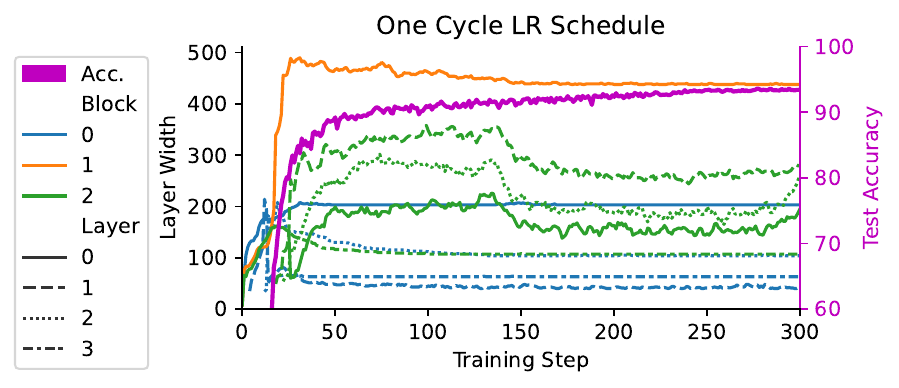}
    \unskip\ \vrule\
    \includegraphics[width= 0.48\textwidth]{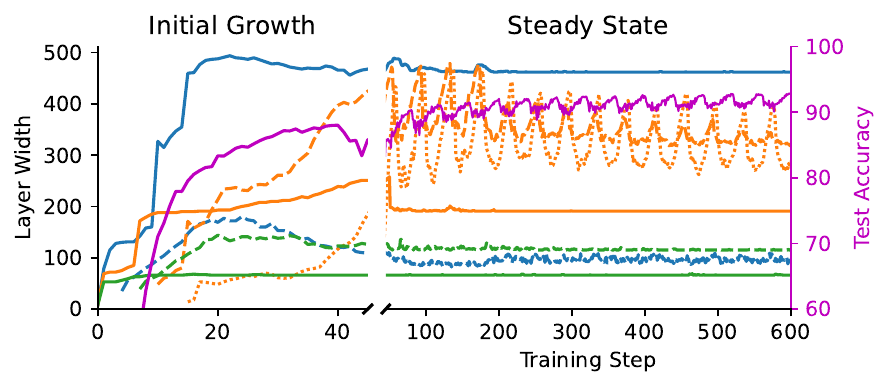}
    \vskip -0.1in
    \caption{SENN stabilizes easily, even when perturbed by a cyclic cosine learning rate schedule.
    As can be observed in the last two layers of block 1, the network capacity expands and compresses in correspondence with the learning rate.}
    \label{fig:cifar10}
    \vspace*{-0.25em}
\end{figure*}

\begin{figure}
    \centering
    \includegraphics[width=0.8\columnwidth]{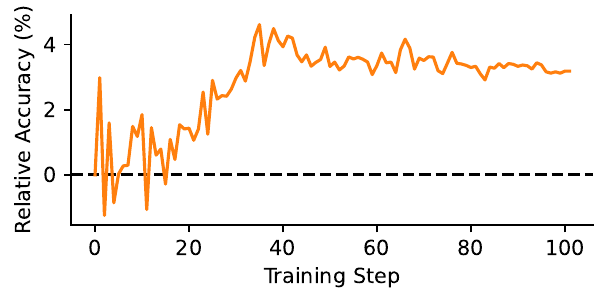}
    \vskip -0.1in
    \caption{Relative test improvements when SENN is used on a CIFAR pre-trained DenseNet in transfer to tiny-imagenet.}
    \label{fig:timnet}
    \vspace*{-0.25em}
\end{figure}

We now examine the full ability of SENN to choose an appropriate architecture through both modification of width and depth in image classification. 
First, the two leftmost panels of figure \ref{fig:mnist} show a fully-connected SENN's total hidden size and validation accuracy during training on the MNIST \cite{deng2012mnist} dataset as a function of total mini-batches seen.
Our SENN is initialized with a single hidden layer of size 10, and promptly adds a second hidden layer, also of size 10.
All five seeds considered then proceed to consistently add width to these layers at a moderate rate until a total hidden size of around 40 is reached,
at which point far fewer productive extensions of the network are found and addition slows dramatically.
This results in respectable validation performance (>97\%) by the end of training with very modest hidden neuron counts (50-60).
To complement this result, we further examine the way in which SENNs adapt their final converged size to the amount of information in the dataset.
To this end, we take class-balanced subsets of MNIST of varying sizes and show SENNs'
converged hidden sizes in the rightmost panel of figure \ref{fig:mnist}.
We can now distinguish three regimes.
For the smallest subsets, the initial hidden size of 10 is sufficient.
For subsets between 10\% and 60\%,
the final hidden size increases logarithmically,
but past that point further increases in subset size do not similarly increase the final network size.
We posit that this is due to substantial redundancy within the MNIST training set,
leaving further capacity growth unnecessary. Thus, SENN does not only provide desirable any time performance, but also tailors its size suitably to the available data.

Finally, the left half of figure \ref{fig:cifar10} shows an analogous classification experiment on the more complex CIFAR-10 dataset \cite{cifar10}. Here, we now use a convolutional SENN and following typical DenseNet convention three blocks are shown, with layers being insertable in any of the blocks and all widths being expandable.
We confirm that in this more challenging scenario, SENN finds a suitable architecture, achieves considerable performance (>93\%), and converges in size to reasonable complexity. In fact, we again emphasize that our method produces strong anytime performance: we are able to continually expand size, and even insert layers, during training without any attendant drops in validation accuracy - a property not shared by methods which rely on reinitializing a new network, e.g. \cite{DBLP:conf/nips/GeifmanE19}. Moreso, upon careful consideration we can observe that SENN can even compress again (reverse addition), if learning permits. We proceed to investigate this unique ability further in the next subsection.

\subsection{SENN stability in growing and pruning cycles}
To shed light on SENN's ability to both expand and compress architectures, we show a variant of our CIFAR-10 classification experiment in the right panel of figure \ref{fig:cifar10}. Here, the main difference is a deliberate use of a cyclical cosine learning rate (LR) schedule ranging between $3\times10^{-4}$ and $3 \times 10^{-8}$ for many epochs. Such repeated LR cycle introduces continuous perturbation to the network. However, after an initial growth cycle  (epochs 0-40), SENN is able to adaptively compensate these perturbations and maintain high accuracy at all times (90-93\%). It retains a fully stable configuration in most of its layers but reacts in correspondence to the learning rate by equally mirroring the cycle through expansion and pruning, as clearly observable in the last two layers of block 1. We posit that this behavior does not only partially absolve us from cumbersome learning rate selection, but further exposes potentially rich connections to neural network learning regimes. In fact, we hypothesize that there exist immediate connections to the distribution and compression of information proposed by information bottleneck theory \cite{Shwartz-Ziv2017,Saxe2018}, but leave in-depth analysis to future work.

\subsection{Pre-trained model transfer improvements}
Our prior expansion and compression insights now prompt a final experiment to highlight SENN's utility in scenarios that start from a pre-trained model. To this end, we now begin with a standard NN model pre-trained to 90+\% on CIFAR-10 (a ``small'' 4-block DenseNet with 128 conv features in each of 3 layers per block) and continue to learn on the more challenging tiny-ImageNet dataset \cite{imagenet,tiny_imagenet}. In addition to standard transfer learning, we then also apply SENN to allow architecture modifications on the pre-trained model.
As the latter now allows capacity to be added as necessary,~we~observe a $\sim$ 3-4\% relative increase in test accuracy over the standard transfer-learned NN ($48.2\%$) in figure \ref{fig:timnet}.
We thus see that SENN allows the adaption of initially suitable trained networks to future tasks with unforeseen capacity requirements.

\section{Conclusion}
We have introduced the natural expansion score $\eta$ and shown how it may be used to cohesively answer the three key questions \when, \where and \what of dynamically growing and pruning neural networks.
We have demonstrated its ability to capture redundancy of new neurons with old and thereby make sensible expansion decisions across time and tasks. This makes SENN a perfect fit not only in finding suitable architectures for novel problems, but also for applications in e.g. continual learning \cite{kprior2021,mundt_survey_2023}.
Prospects for further development are promising, as our theoretical results regarding $\eta$ apply for arbitrary expansions of parameterized models,
and our method of expansion would extend naturally to \eg{} transformers, or normalizing flows
where layers may be initialized invertibly.

\section*{Acknowledgements}
This work was supported by the
project “safeFBDC - Financial Big Data Cluster” (FKZ:
01MK21002K), funded by the German Federal Ministry for
Economics Affairs and Energy as part of the GAIA-x initiative,
and the
Hessian research priority programme LOEWE within the
project “WhiteBox”.\\
We thank Emtiyaz Khan for valuable feedback on an earlier version of this manuscript.

\bibliography{references}

\begin{thebibliography}{37}
\providecommand{\natexlab}[1]{#1}
\providecommand{\url}[1]{\texttt{#1}}
\expandafter\ifx\csname urlstyle\endcsname\relax
  \providecommand{\doi}[1]{doi: #1}\else
  \providecommand{\doi}{doi: \begingroup \urlstyle{rm}\Url}\fi

\bibitem[Amari(1998)]{DBLP:journals/neco/Amari98}
Amari, S.
\newblock Natural gradient works efficiently in learning.
\newblock \emph{Neural Comput.}, 10\penalty0 (2):\penalty0 251--276, 1998.

\bibitem[Ash(1989)]{doi:10.1080/09540098908915647}
Ash, T.
\newblock Dynamic node creation in backpropagation networks.
\newblock \emph{Connection Science}, 1\penalty0 (4):\penalty0 365--375, 1989.

\bibitem[Bradbury et~al.(2018)Bradbury, Frostig, Hawkins, Johnson, Leary, Maclaurin, Necula, Paszke, Vander{P}las, Wanderman-{M}ilne, and Zhang]{jax2018github}
Bradbury, J., Frostig, R., Hawkins, P., Johnson, M.~J., Leary, C., Maclaurin, D., Necula, G., Paszke, A., Vander{P}las, J., Wanderman-{M}ilne, S., and Zhang, Q.
\newblock {JAX}: composable transformations of {P}ython+{N}um{P}y programs.
\newblock 2018.
\newblock URL \url{http://github.com/google/jax}.

\bibitem[Chen et~al.(2016)Chen, Goodfellow, and Shlens]{net2net2016}
Chen, T., Goodfellow, I., and Shlens, J.
\newblock Net2net: Accelerating learning via knowledge transfer.
\newblock \emph{International Conference on Representation Learning (ICLR)}, 2016.

\bibitem[Deng(2012)]{deng2012mnist}
Deng, L.
\newblock The mnist database of handwritten digit images for machine learning research.
\newblock \emph{IEEE Signal Processing Magazine}, 29\penalty0 (6):\penalty0 141--142, 2012.

\bibitem[Draelos et~al.(2017)Draelos, Miner, Lamb, Cox, Vineyard, Carlson, Severa, James, and Aimone]{Draelos2017}
Draelos, T.~J., Miner, N.~E., Lamb, C.~C., Cox, J.~A., Vineyard, C.~M., Carlson, K.~D., Severa, W.~M., James, C.~D., and Aimone, J.~B.
\newblock {Neurogenesis deep learning: Extending deep networks to accommodate new classes}.
\newblock \emph{International Joint Conference on Neural Networks (IJCNN)}, 2017.

\bibitem[Evci et~al.(2022)Evci, van Merrienboer, Unterthiner, Pedregosa, and Vladymyrov]{DBLP:conf/iclr/EvciMUPV22}
Evci, U., van Merrienboer, B., Unterthiner, T., Pedregosa, F., and Vladymyrov, M.
\newblock Gradmax: Growing neural networks using gradient information.
\newblock \emph{International Conference on Learning Representations (ICLR)}, 2022.

\bibitem[Geifman \& El{-}Yaniv(2019)Geifman and El{-}Yaniv]{DBLP:conf/nips/GeifmanE19}
Geifman, Y. and El{-}Yaniv, R.
\newblock Deep active learning with a neural architecture search.
\newblock \emph{Neural Information Processing Systems (NeurIPS)}, 2019.

\bibitem[Girolami \& Calderhead(2011)Girolami and Calderhead]{girolami2011_rmhmc}
Girolami, M. and Calderhead, B.
\newblock {Riemann Manifold Langevin and Hamiltonian Monte Carlo Methods}.
\newblock \emph{Journal of the Royal Statistical Society Series B: Statistical Methodology}, 73\penalty0 (2):\penalty0 123--214, 2011.

\bibitem[Gross(2000)]{Gross2000}
Gross, C.~G.
\newblock {Neurogenesis in the adult brain: Death of a dogma}.
\newblock \emph{Nature Reviews Neuroscience}, 1:\penalty0 67--73, 2000.

\bibitem[Hassibi et~al.(1993)Hassibi, Stork, and Wolff]{optimal_brain_surgeon}
Hassibi, B., Stork, D.~G., and Wolff, G.~J.
\newblock Optimal brain surgeon and general network pruning.
\newblock \emph{International Conference on Neural Networks (ICNN)}, 1993.

\bibitem[Heek et~al.(2023)Heek, Levskaya, Oliver, Ritter, Rondepierre, Steiner, and van {Z}ee]{flax2020github}
Heek, J., Levskaya, A., Oliver, A., Ritter, M., Rondepierre, B., Steiner, A., and van {Z}ee, M.
\newblock {F}lax: A neural network library and ecosystem for {JAX}.
\newblock 2023.
\newblock URL \url{http://github.com/google/flax}.

\bibitem[Huang et~al.(2019)Huang, Cheng, Bapna, Firat, Chen, Chen, Lee, Ngiam, Le, Wu, and Chen]{DBLP:conf/nips/HuangCBFCCLNLWC19}
Huang, Y., Cheng, Y., Bapna, A., Firat, O., Chen, D., Chen, M.~X., Lee, H., Ngiam, J., Le, Q.~V., Wu, Y., and Chen, Z.
\newblock Gpipe: Efficient training of giant neural networks using pipeline parallelism.
\newblock 2019.

\bibitem[Khan \& Swaroop(2021)Khan and Swaroop]{kprior2021}
Khan, M.~E. and Swaroop, S.
\newblock Knowledge-adaptation priors.
\newblock \emph{Neural Information Processing Systems (NeurIPS)}, 2021.

\bibitem[Krizhevsky et~al.(2009)Krizhevsky, Nair, and Hinton]{cifar10}
Krizhevsky, A., Nair, V., and Hinton, G.
\newblock Cifar-10 (canadian institute for advanced research).
\newblock \emph{Technical Report}, 2009.
\newblock URL \url{http://www.cs.toronto.edu/~kriz/cifar.html}.

\bibitem[Krizhevsky et~al.(2017)Krizhevsky, Sutskever, and Hinton]{DBLP:journals/cacm/KrizhevskySH17}
Krizhevsky, A., Sutskever, I., and Hinton, G.~E.
\newblock Imagenet classification with deep convolutional neural networks.
\newblock \emph{ACM Communications}, 60\penalty0 (6):\penalty0 84--90, 2017.

\bibitem[Kudithipudi et~al.(2022)Kudithipudi, Aguilar-Simon, Babb, Bazhenov, Blackiston, Bongard, Brna, Chakravarthi~Raja, Cheney, Clune, Daram, Fusi, Helfer, Kay, Ketz, Kira, Kolouri, Krichmar, Kriegman, and Siegelmann]{Kudithipudi2022}
Kudithipudi, D., Aguilar-Simon, M., Babb, J., Bazhenov, M., Blackiston, D., Bongard, J., Brna, A., Chakravarthi~Raja, S., Cheney, N., Clune, J., Daram, A., Fusi, S., Helfer, P., Kay, L., Ketz, N., Kira, Z., Kolouri, S., Krichmar, J., Kriegman, S., and Siegelmann, H.
\newblock Biological underpinnings for lifelong learning machines.
\newblock \emph{Nature Machine Intelligence}, 4:\penalty0 196--210, 03 2022.

\bibitem[Le \& Yang(2015)Le and Yang]{tiny_imagenet}
Le, Y. and Yang, X.~S.
\newblock Tiny imagenet visual recognition challenge.
\newblock 2015.
\newblock URL \url{https://api.semanticscholar.org/CorpusID:16664790}.

\bibitem[LeCun et~al.(1989)LeCun, Denker, and Solla]{optimal_brain_damage}
LeCun, Y., Denker, J.~S., and Solla, S.~A.
\newblock Optimal brain damage.
\newblock 1989.

\bibitem[LeCun et~al.(2012)LeCun, Bottou, Orr, and M{\"u}ller]{LeCun2012}
LeCun, Y.~A., Bottou, L., Orr, G.~B., and M{\"u}ller, K.-R.
\newblock \emph{Efficient BackProp}, pp.\  9--48.
\newblock Springer Berlin Heidelberg, 2012.

\bibitem[Maile et~al.(2022)Maile, Rachelson, Luga, and Wilson]{maile2022}
Maile, K., Rachelson, E., Luga, H., and Wilson, D.~G.
\newblock When, where, and how to add new neurons to anns.
\newblock \emph{International Conference on Automated Machine Learning (AutoML)}, 2022.

\bibitem[Martens(2020)]{DBLP:journals/jmlr/Martens20}
Martens, J.
\newblock New insights and perspectives on the natural gradient method.
\newblock \emph{Journal of Machine Learning Research (JMLR)}, 21:\penalty0 146:1--146:76, 2020.

\bibitem[Martens \& Grosse(2015)Martens and Grosse]{DBLP:conf/icml/MartensG15}
Martens, J. and Grosse, R.~B.
\newblock Optimizing neural networks with kronecker-factored approximate curvature.
\newblock \emph{International Conference on Machine Learning (ICML)}, 2015.

\bibitem[Molina et~al.(2020)Molina, Schramowski, and Kersting]{DBLP:conf/iclr/0001SK20}
Molina, A., Schramowski, P., and Kersting, K.
\newblock Pad{\'{e}} activation units: End-to-end learning of flexible activation functions in deep networks.
\newblock \emph{International Conference on Learning Representations (ICLR)}, 2020.

\bibitem[Mundt et~al.(2023)Mundt, Hong, Pliushch, and Ramesh]{mundt_survey_2023}
Mundt, M., Hong, Y., Pliushch, I., and Ramesh, V.
\newblock A wholistic view of continual learning with deep neural networks: Forgotten lessons and the bridge to active and open world learning.
\newblock \emph{Neural Networks}, 160\penalty0 (03):\penalty0 306--336, 2023.

\bibitem[Nakkiran et~al.(2020)Nakkiran, Kaplun, Bansal, Yang, Barak, and Sutskever]{DBLP:conf/iclr/NakkiranKBYBS20}
Nakkiran, P., Kaplun, G., Bansal, Y., Yang, T., Barak, B., and Sutskever, I.
\newblock Deep double descent: Where bigger models and more data hurt.
\newblock \emph{International Conference on Learning Representations (ICLR)}, 2020.

\bibitem[Pedregosa et~al.(2011)Pedregosa, Varoquaux, Gramfort, Michel, Thirion, Grisel, Blondel, Prettenhofer, Weiss, Dubourg, Vanderplas, Passos, Cournapeau, Brucher, Perrot, and Duchesnay]{scikit-learn}
Pedregosa, F., Varoquaux, G., Gramfort, A., Michel, V., Thirion, B., Grisel, O., Blondel, M., Prettenhofer, P., Weiss, R., Dubourg, V., Vanderplas, J., Passos, A., Cournapeau, D., Brucher, M., Perrot, M., and Duchesnay, E.
\newblock Scikit-learn: Machine learning in {P}ython.
\newblock \emph{Journal of Machine Learning Research (JMLR)}, 12:\penalty0 2825--2830, 2011.

\bibitem[Russakovsky et~al.(2015)Russakovsky, Deng, Su, Krause, Satheesh, Ma, Huang, Karpathy, Khosla, Bernstein, Berg, and Fei{-}Fei]{imagenet}
Russakovsky, O., Deng, J., Su, H., Krause, J., Satheesh, S., Ma, S., Huang, Z., Karpathy, A., Khosla, A., Bernstein, M.~S., Berg, A.~C., and Fei{-}Fei, L.
\newblock Imagenet large scale visual recognition challenge.
\newblock \emph{International Journal of Computer Vision (IJCV)}, 115\penalty0 (3):\penalty0 211--252, 2015.

\bibitem[Rusu et~al.(2016)Rusu, Rabinowitz, Desjardins, Soyer, Kirkpatrick, Kavukcuoglu, Pascanu, and Hadsell]{DBLP:journals/corr/RusuRDSKKPH16}
Rusu, A.~A., Rabinowitz, N.~C., Desjardins, G., Soyer, H., Kirkpatrick, J., Kavukcuoglu, K., Pascanu, R., and Hadsell, R.
\newblock Progressive neural networks.
\newblock \emph{arXiv:1606.04671}, 2016.

\bibitem[Saxe et~al.(2018)Saxe, Bansal, Dapello, Advani, Kolchinsky, Tracey, and Cox]{Saxe2018}
Saxe, A.~M., Bansal, Y., Dapello, J., Advani, M., Kolchinsky, A., Tracey, B.~D., and Cox, D.~D.
\newblock On the information bottleneck theory of deep learning.
\newblock \emph{International Conference on Representation Learning (ICLR)}, 2018.

\bibitem[Shwartz-Ziv \& Tishby(2017)Shwartz-Ziv and Tishby]{Shwartz-Ziv2017}
Shwartz-Ziv, R. and Tishby, N.
\newblock Opening the black box of deep neural networks via information.
\newblock \emph{arXiv:1703.00810}, 2017.

\bibitem[Szegedy et~al.(2015)Szegedy, Liu, Jia, Sermanet, Reed, Anguelov, Erhan, Vanhoucke, and Rabinovich]{DBLP:conf/cvpr/SzegedyLJSRAEVR15}
Szegedy, C., Liu, W., Jia, Y., Sermanet, P., Reed, S.~E., Anguelov, D., Erhan, D., Vanhoucke, V., and Rabinovich, A.
\newblock Going deeper with convolutions.
\newblock 2015.

\bibitem[Vadodaria \& Jessberger(2014)Vadodaria and Jessberger]{Vadodaria2014}
Vadodaria, K.~C. and Jessberger, S.
\newblock {Functional neurogenesis in the adult hippocampus: Then and now}.
\newblock \emph{Frontiers in Neuroscience}, 8:\penalty0 1--3, 2014.

\bibitem[von Mises(1964)]{von_mises1964}
von Mises, R.
\newblock \emph{Mathematical theory of probability and statistics}, chapter VIII.9.3.
\newblock Academic Press, New York, 1964.

\bibitem[Wu et~al.(2019)Wu, Wang, and Liu]{DBLP:conf/nips/WuW019}
Wu, L., Wang, D., and Liu, Q.
\newblock Splitting steepest descent for growing neural architectures.
\newblock \emph{Neural Information Processing Systems (NeurIPS)}, 2019.

\bibitem[Wu et~al.(2020)Wu, Liu, Stone, and Liu]{DBLP:conf/nips/WuLS020}
Wu, L., Liu, B., Stone, P., and Liu, Q.
\newblock Firefly neural architecture descent: a general approach for growing neural networks.
\newblock \emph{Neural Information Processing Systems (NeurIPS)}, 2020.

\bibitem[Yoon et~al.(2018)Yoon, Yang, Lee, and Hwang]{DBLP:conf/iclr/YoonYLH18}
Yoon, J., Yang, E., Lee, J., and Hwang, S.~J.
\newblock Lifelong learning with dynamically expandable networks.
\newblock \emph{International Conference on Learning Representations (ICLR)}, 2018.

\end{thebibliography}
\bibliographystyle{icml2024}

\newpage
\appendix
\onecolumn

\section{Proofs}
\label{app:proofs}
\setcounter{theorem}{0}

\subsection{Theorem 1: Bounded rate of addition}
In this section we prove theorem 1 of the main body.
We will assume $\mF \succ 0$ to be positive definite, with the following straightforward consequence
\begin{lemma}
\label{lemma:positive_eta}
The natural expansion score is non-negative $\eta = \vg^T \inv{\mF} \vg \geq 0$.
\end{lemma}
\begin{proof}
If $\mF \succ 0$, then $\inv{\mF} \succ 0$,
and $\vv^T \inv{\mF} \vv \geq 0$ for all $\vv$.
\end{proof}

Considering the effect of the expansion threshold $\tau$ we obtain the following bound:
\begin{lemma}
\label{lemma:tau_rate_bound}
Let $\eta$ have initial value $\eta_0$ and be bounded above by $\lambda \geq \eta$.
If the threshold $\tau$ guarantees that $\eta_i > (1 + \tau) \eta_{i-1}$ for the $i$-th addition,
then the maximum number of successive additions $N_s$ is bounded by
$N_s < \frac{\ln \lambda - \ln \eta_0}{\ln (1 + \tau)}$.
\end{lemma}
\begin{proof}
Due to the threshold $\tau$, $\eta$ grows at least exponentially: $\eta_i > (1 + \tau)^i \eta_0$.
But $\eta$ is bounded: $\lambda \geq \eta_i > (1 + \tau)^i \eta_0$.
Since $\ln$ is monotonic, we may take logarithms:
$\ln \lambda > i \ln (1 + \tau) + \ln \eta_0$.
and rearrange to get $i < \frac{\ln \lambda - \ln \eta_0}{\ln (1 + \tau)}$ for all additions $i$.
This true for every $i$-th addition which is accepted, and so in particular also true for the last $N_s$-th addition.
\end{proof}

Considering also the effect of the stopping criterion $\alpha$ we obtain theorem 1:
\begin{theorem}
\label{theorem:rate_bound}
If the stopping criterion $\alpha$ guarantees that $\eta_{i} - \eta_{i-1} > \alpha$,
then the maximum number of successive additions $N_s$ is either $0$, or bounded by
$N_s < 1 + \frac{\ln \lambda - \ln \alpha}{\ln (1 + \tau)}$.
\end{theorem}
\begin{proof}
Either $N_s = 0$, or there is a first addition with natural expansion score $\eta_1$ for which
$\eta_1 - \eta_0 > \alpha$.
From lemma \ref{lemma:positive_eta} we then have $\eta_1 > \alpha$.
We may then substitute $\alpha$ into lemma \ref{lemma:tau_rate_bound} in place of $\eta_0$ to obtain
a bound on further additions, yielding
$N_s < 1 + \frac{\ln \lambda - \ln \alpha}{\ln (1 + \tau)}$.
\end{proof}



This theorem is important because it guarantees that SENN will add a limited number of neurons or layers before continuing training.
Intuitively, this is because it rapidly becomes the case that any new neuron is either not relevant to rapidly decreasing the loss, or is redundant with some already extant neuron.

\subsection{ Theorem 2: Lower bound on increase in natural expansion score}
We now prove theorem 2 of the main body, concerning a lower bound on the increase in natural expansion score $\eta$ due to the addition of new proposed neuron(s) to a layer.
Let the joint activations $\va = \begin{bsmallmatrix} \va_c \\ \va_p \end{bsmallmatrix}$ of the current and proposed neurons have second moment
$\expect{\va \va^T} = \mA = \begin{bsmallmatrix} \mA_{c} & \mC_{cp} \\ \mC_{pc} & \mA_{p} \end{bsmallmatrix}$.
We will assume the Fisher matrix $\mF$ for the layer to which neurons are to be added to factorize as $\mF = \mS \otimes \mA$, where $\mS \succ 0$ is positive definite.
We first derive a convenient form of a known result discussed in, for example, \citet{von_mises1964},
related to the joint covariance of multivariate Gaussian distributions.
\begin{lemma}
\label{lemma:block_residual_inverse}
Let $\hat{\mA}_p = \mA_p - \mC_{pc} \inv{\mA}_c \mC_{cp}$ be the Schur complement of $\mA_c$ in $\mA$.
Let also $\vv = \begin{bsmallmatrix} \vv_c \\ \vv_p \end{bsmallmatrix}$ be an arbitrary vector,
and $\mR$ be the linear operator defined by $\mR \vv = \vv_p - \mC_{pc} \inv{\mA}_c \vv_c$,
i.e. the residual part of $\vv_p$ not predicted by $\vv_c$.
Then, $
\vv^T \inv{\mA} \vv = \vv_c^T \inv{\mA}_c \vv_c +
 (\mR \vv)^T \inv{\hat{\mA}}_p \mR \vv
$.
\end{lemma}

\begin{proof}
The following may be obtained by performing a block LDU decomposition:
\begin{equation}
\mA =
\begin{bmatrix}
\mA_{c} & \mC_{cp} \\ \mC_{pc} & \mA_{p}
\end{bmatrix} = 
\begin{bmatrix}
\mI_c & 0 \\
\mC_{pc} \inv{\mA}_c & \mI_p
\end{bmatrix}
\begin{bmatrix}
\mA_c & 0 \\
0 & \hat{\mA}_p
\end{bmatrix}
\begin{bmatrix}
\mI_c & \inv{\mA}_c \mC_{cp} \\
0 & \mI_p
\end{bmatrix}
\end{equation}
which we may then use to decompose $\inv{\mA}$:
\begin{equation}
\inv{\mA} =
\begin{bmatrix}
\mA_{c} & \mC_{cp} \\ \mC_{pc} & \mA_{p}
\end{bmatrix}^{-1} = 
\begin{bmatrix}
\mI_c & -\inv{\mA}_c \mC_{cp} \\
0 & \mI_p
\end{bmatrix}
\begin{bmatrix}
\mA_c & 0 \\
0 & \hat{\mA}_p
\end{bmatrix}
\begin{bmatrix}
\mI_c & 0 \\
-\mC_{pc} \inv{\mA}_c & \mI_p
\end{bmatrix}
\end{equation}
The desired result then follows by substitution into $\vv^T \inv{\mA} \vv$:
\begin{equation}
\begin{split}
\vv^T \inv{\mA} \vv & =
\begin{bmatrix} \vv_c & \vv_p \end{bmatrix}
\begin{bmatrix}
\mI_c & -\inv{\mA}_c \mC_{cp} \\
0 & \mI_p
\end{bmatrix}
\begin{bmatrix}
\mA_c^{-1} & 0 \\
0 & \hat{\mA}_p^{-1}
\end{bmatrix}
\begin{bmatrix}
\mI_c & 0 \\
-\mC_{pc} \inv{\mA}_c & \mI_p
\end{bmatrix}
\begin{bmatrix} \vv_c \\ \vv_p \end{bmatrix} \\ & =
\vv_c^T \inv{\mA}_c \vv_c + 
(\vv_p - \mC_{pc} \inv{\mA}_c \vv_c)^T \hat{\mA}_p^{-1} (\vv_p - \mC_{pc} \inv{\mA}_c \vv_c)
\end{split}
\end{equation}

\end{proof}

Recall from section 3.6 that $\eta$ may be expressed as a trace:
$\eta = \Tr[\mS^{-1} \expect{\vg \va^T} \mA^{-1} \expect{\va \vg^T}]$
where $\vg$ is the derivative of the loss with respect to the outputs (i.e. layer pre-activations) of the linear transform.
\begin{corollary}
\label{corollary:schur_form_eta}
We can use lemma \ref{lemma:block_residual_inverse} to write the increase in natural expansion score $\Delta \eta$ as
\begin{equation}
    \begin{split}
    \Delta \eta & = \Tr[\mS^{-1} \expect{\vg \va^T} \mA^{-1} \expect{\va \vg^T}]
    - \Tr[\mS^{-1} \expect{\vg \va_c} \mA_c^{-1} \expect{\va_c \vg^T}] \\
    & = \Tr[\mS^{-1} \expect{\vg (\mR \va)^T} \hat{\mA}_p^{-1} \expect{(\mR \va) \vg^T}]
    \end{split}
\end{equation}
where we can take $\mR$ inside the expectations by linearity.
\end{corollary}


It is computationally convenient for us to be able to have an expression in terms of residual gradients instead of residual activations, so we note the following:
\begin{lemma}
\label{lemma:residual_interchange}
$\expect{\vg (\mR \va)^T} = \expect{\vg_r \va_p^T}$
where $\vg_r = \vg - \expect{\vg \va_c^T} \mA_c^{-1} \va_c$ is the residual gradient.
\end{lemma}
\begin{proof}
\begin{align*}
    \expect{\vg (\mR \va)^T} & = \expect{\vg (\va_p - \expect{\va_p \va_c^T} \mA_c^{-1} \va_c)^T} \\
    & = \expect{\vg \va_p^T} - \expect{\vg \va_c^T} \mA_c^{-1} \expect{\va_c \va_p^T} \\
    & = \expect{(\vg - \expect{\vg \va_c} \mA_c^{-1} \va_c) \va_p^T} \\
    & = \expect{\vg_r \va_p^T}
\end{align*}
\end{proof}

Finally, we establish the following relationship between $\mA_p^{-1}$ and $\hat{\mA}_p^{-1}$:
\begin{lemma}
\label{lemma:schur_complement_succeq}
$\hat{\mA}_p^{-1} - \mA_p^{-1}= (\mA_p - \mC_{pc} \inv{\mA}_c \mC_{cp})^{-1} - \mA_p^{-1} \succeq 0$.
\end{lemma}
\begin{proof}
The matrix inverse $\hat{\mA}_p^{-1}$ can be expanded as the following power series
\begin{equation}
    \hat{\mA}_p^{-1} = (\mA_p - \mC_{pc} \inv{\mA}_c \mC_{cp})^{-1} = 
    \sum_{n=0}^\infty \inv{\mA}_p (\mC_{pc} \inv{\mA}_c \mC_{cp} \inv{\mA}_p)^n
\end{equation}
We observe that this is a sum of positive semi-definite matrices, and truncate the series at $n=0$ and rearrange:
\begin{equation}
    \hat{\mA}_p^{-1} - \inv{\mA}_p = \sum_{n=1}^\infty \inv{\mA}_p (\mC_{pc} \inv{\mA}_c \mC_{cp} \inv{\mA}_p)^n \succeq 0
\end{equation}
\end{proof}

We may now prove the main body's theorem 2.
\begin{theorem}
$\Delta \eta'$ is a lower bound on the increase in natural expansion score $\Delta \eta$ due to the addition of some proposed neurons $p$:
\begin{equation}
    \Delta \eta
    \geq \Delta \eta' = \Tr[\mS^{-1} \expect{\vg_r \va_p^T} \mA_p^{-1} \expect{\va_p \vg_r^T}]
\end{equation}
\end{theorem}
\begin{proof}
Substituting lemma \ref{lemma:residual_interchange} into corollary \ref{corollary:schur_form_eta} we have
$\Delta \eta = \Tr[\mS^{-1} \expect{\vg_r \va^T} \hat{\mA}_p^{-1} \expect{\va \vg_r^T}]$.
The difference between $\Delta \eta$ and $\Delta \eta'$ is given by
$\Delta \eta - \Delta \eta' = \Tr[\mS^{-1} \expect{\vg_r \va^T} (\hat{\mA}_p^{-1} - \mA_p^{-1}) \expect{\va \vg_r^T}]$.
This is the squared norm of $\expect{\vg_r \va^T}$ as a vector according to the Kronecker product
$\mS^{-1} \otimes (\hat{\mA}_p^{-1} - \mA_p^{-1})$.
The first factor is positive semi-definite by assumption, the second by lemma \ref{lemma:schur_complement_succeq},
and the Kronecker product of positive semi-definite matrices is positive semi-definite.
Therefore $\Delta \eta - \Delta \eta' \geq 0$ and so $\Delta \eta \geq \Delta \eta'$.
\end{proof}
The significance of this lower bound on $\Delta \eta$ is that $\vg_r$ and $\mS^{-1}$ may be computed once,
and then used to optimize very many proposals with different activations $\va_p$.
That is, performing $N$ steps of gradient descent to optimize proposed neurons $p$ scales linearly in the evaluation cost of $\va_p$ and $\mA_p^{-1}$.
These linear costs are unaffected by the number of neurons currently in the layer being added to,
and unaffected by the total number of layers in the network.






\section{Hyperparameters and experimental details}

All experiments were run on a single Nvidia A100 or V100 GPU, using no longer than two days each, except for the cyclic learning rate cifar10 experiment which we intentionally let run for longer than necessary.
Our implementation uses the JAX \cite{jax2018github} autodifferentiation and Flax \cite{flax2020github} neural network libraries.
The full source code used to run the experiments is provided in the supplementary material, and will be made publicly available on publication of this work.
In all MLP experiments we optimize our parameters via natural gradient descent with a learning rate of 0.1 and Tikhonov damping of magnitude 0.1.
In the MNIST classification experiments we use batches of size 1024 and a weight decay of rate 0.001.
We initialize our dense layers with the default initialization of Flax (LeCun Normal) \cite{LeCun2012},
and use a unit normal initialization for the parameters of our rational functions.

For the visualization experiments we use $\tau=1$, for the image classification experiments we use $\tau=$7e-3 and $\tau=$3e-2
for the whole dataset and variable subset experiments respectively.
Larger thresholds $\tau$ result in longer training times but more conservative network sizes and higher accuracy of $\eta$ estimates due to $\mF$ being a closer approximation to the curvature near convergence on the existing parameters.
Any extra costs are negligible for the visualization experiments, so we use the intuitive value of $1$,
but we choose $\tau$ values for the image classification experiments in light of this natural trade-off.
We use $\alpha=0.0025$ for all MLP experiments apart from the whole dataset MNIST classification, for which we use $\alpha=0.25$.
Here the latter choice compensates for larger noise in $\Delta \eta'$ introduced by use of a validation batch, as will be discussed shortly.
We adjust the expansion score increases for layer additions by a constant factor of $2$ in the visualization experiments and $60$ in the MNIST classification experiments.
These values are selected to be within an order of magnitude of the actual layer sizes expected in classification of a toy dataset versus MNIST, and so of the number of new neurons a new layer represents.

We calculate the natural gradient via the conjugate gradient method with a maximum iteration count of 100 when optimizing the existing parameters.
When optimizing the initializations of proposed neurons or layers we use the Kronecker factored approximation of the Fisher matrix for the relevant layer
based on derivatives of the predictions of the network as in \citet{DBLP:conf/icml/MartensG15}.
We compute $\Delta \eta'$ based on this $\Tilde{\mF}$ and normalize it with respect to the output gradient magnitudes of the particular task.
When comparing $\Delta \eta' / \eta_c$ to $\tau$ we use the $\eta_c$ value given by $\Tilde{\mF}$ for the layer in question.
When considering adding layers, we ensure new layers are invertible by adding a regularization term of $0.01(\ln\det \mW)^2$ when optimizing the initialization of their linear transform $\mW$,
and by setting the minimal singular values of $\mW$ to be at least 0.001 times its average singular value before adding the layer to the network.
In our visualization experiments we do not use batching, so we consider adding depth and width every 30 steps,
and add at most one layer per 90 steps.
In the MNIST classification experiments we use batching and so consider adding width and depth every 10 epochs,
adding at most one layer each time.

We use the same scheme for initializing proposed new neurons or layers as for initializing the starting network.
In our whole dataset MNIST classification experiment we then optimize proposal initializations to maximize $\Delta \eta'$ via 300 steps of vanilla gradient descent
on a fixed batch of 1024 images.
We consider 10000 neuron proposals and 100 layer proposals per location, and use a learning rate of 0.3,
reducing this by a factor of 3 as necessary to maintain monotonic improvement in $\Delta \eta'$ for each proposal.
We take the best proposal on this batch of size 1024 for each depth and width addition location,
and reevaluate its $\Delta \eta'$ on a fixed validation batch of size 1024 when deciding whether and where to add.
The variable degree of overfitting of the best proposal results in some noise in $\Delta \eta'$ at each location which we compensate for by choosing a relatively large $\alpha$.

For our other MLP experiments we optimize proposal initializations using 3000 steps of the Metropolis Adjusted Langevin Algorithm (MALA) \cite{girolami2011_rmhmc},
using a unit gaussian prior on initializations during these steps.
We use a temperature $T$ of 10 and an initial step size of 0.3, and adjust by a factor of 3 every 10 steps if necessary to maintain an acceptance rate of around 0.6.
We consider 100 width proposals and 100 layer proposals for each location,
and obtain 100 final MALA samples $i$ for each location width could be added and each location depth could be added.
We then construct a categorical distribution over each set of 100 samples via $\softmax (\frac{1}{T}\Delta \eta_i')$,
and use the corresponding expectation of $\Delta \eta'$ when deciding when and where to add capacity and whether it should be depth or width.
We draw initializations for new capacity from this categorical distribution,
except in the initial least squares regression experiment, where we use $\argmax_i \Delta \eta_i'$ over the 100 samples $i$ to make figure 2 more intuitive.

For the CNN experiments we use the Adam optimizer with a weight decay of 1e-2 and standard hyperparameters of $\beta_1=$1e-1 and $\beta_2=$1e-2 instead of natural gradient descent when optimizing the parameters for ease of comparison with existing results.
We use mini-batches of size 64, again for ease of comparison.
We use a standard Optax cosine annealing learning rate schedule which peaks at 10\% of each cycle with a learning rate of 3e-4 and ends at 3e-8.
We use a single cycle of 300 epochs in the first cifar10 experiment and 15 cycles of 40 epochs each for the second cyclic learning rate cifar10 experiment.
When pre-training on cifar10 before transfer learning we use 3 cycles of 40 epochs each, and when transferring to Tiny-Imagenet we use a single cycle of 100 epochs.

We use data augmentation for cifar10 and Tiny-Imagenet, after scaling the images to 32x32 for the pure cifar10 experiments, or 64x64 for the transfer learning experiments.
50\% of the time we use the image unchanged, and 50\% we augment it. If we augment it we follow the following procedure:
We first pad images to twice their original size, then randomly rotate by up to 20 degrees with probability 30\%, then randomly scale by 0.8-1.2 with probability 30\%, then finally do a random crop back to the original size (effectively including a random translation).
With probability 30\% we then augment the pixel values, changing saturation and contrast by 0.8-1.3 and changing brightness by up to 0.3.

For CNNs we use a block-diagonal Kronecker factored approximation \cite{DBLP:conf/icml/MartensG15} to a modified Generalized Gauss-Newton curvature approximation as described in appendix \ref{app:ubah}.
To facilitate this approximation we use swish activations functions over ReLU as they are second differentiable but otherwise similar.
We track the inverse curvature throughout training via rank one updates and a stochastic estimator as described in appendix \ref{app:rank_one_updates}.
(In places this requires a moving average, the exact time periods of which can be found in the attached source code.)
As JAX is just-in-time compiled dependent on parameter sizes, we allocate all layers with initial widths of $128$ features, only $8$ of which are initially actively used by the network when training from scratch with SENN.
We reuse the remaining features as proposals for new neurons, tracking their gradients and correlations with existing neurons as part of our standard curvature tracking code.
We initialize all neurons, used and unused, with the default Flax distribution (LeCun Normal), and reinitialize each unused neuron, with probability $10\%$ independently, at the end of each epoch.
We use the lower bound on $\Delta \eta$ when deciding whether to activate neurons, and use an expansion threshold of $\tau =$3e-4 for cifar10 and $\tau=$1e-4 for Tiny-Imagenet.
We activate neurons surpassing the expansion threshold with a probability of 1e-3 independently at every step, \ie{} in expectation after around one epoch.
We deactivate neurons whose removal cost drops below this same threshold at the same rate of 1e-3 per step.
This removal cost is calculated for an active neuron by using the curvature estimate to predict the gradient the resulting zeroed output weights would have were we to prune this neuron and optimally compensate with existing weights.
We then evaluate the increase in expansion score that would result from re-enabling this neuron in this counterfactual modified NN and call this counterfactual increase the ``removal cost'' of the neuron.
Clearly if we would immediately re-add the neuron were it to be pruned, we should not prune it, which is why it makes sense to compare this ``removal cost'' to the expansion threshold, and hence the title of section \ref{sec:pruning}.
We set the stopping criterion $\alpha$ to zero, as it was not needed in our CNN experiments - noise in the gradient estimates of the Adam optimizer is such that gradients do not approach zero excessively quickly.
When more than $50\%$ of neurons in a layer are used, we recompile the network with this layer $50\%$ larger with more unused neurons.
We prepend an unused layer of width $64$ before every active layer, and add this layer to the network when the average expansion score for its constituent neurons exceeds $3$ times the width threshold $\tau$.

\section{The consequences of non-constant curvature for total neurons added}





In section 3.4 we discussed the total number of neurons added during training, and in particular the extent to which we could provide bounds on this.
In the case where the Fisher $\mF$ is constant over training and exactly equal to the hessian,
the dynamics of training are very simple.
The loss $L$ has its global minimum at the point reached by a step of exactly $\mF^{-1} \vg$,
and it can be seen by integration that the reduction in loss due to such a step is exactly $\Delta L = \frac{1}{2} \vg^T \mF^{-1} \vg = \frac{1}{2}\eta$.
The stopping criterion $\alpha$ corresponds to the requirement that parameter expansions should enable a further reduction in loss of at least $\frac{1}{2} \alpha$.
Since $\eta \leq \lambda$ is bounded by $\lambda$, the maximum possible reduction in loss is $\Delta L_\text{max} = \frac{1}{2} \lambda$.
If we pessimistically assume that every parameter expansion enables the minimal loss reduction of only $\frac{1}{2} \alpha$,
then the total number of added neurons $N_T$ is still bounded by $N_T < \frac{\lambda}{\alpha}$.

The case where the true hessian of the loss $\mH$ is some constant multiple of the Fisher $\mH = \kappa \mF$ which is itself constant,
is almost as simple.
The parameters evolve along the same trajectory, only they move a factor of $\kappa$ faster than they would if $\mF = \mH$.
This also results in a rescaling $\eta = \kappa \eta_B$ of natural expansion scores relative to the baseline value $\eta_B$ in the case where $\mF$ was accurate.
While this has no effect on the behaviour of the expansion threshold $\tau$,
the inflated $\eta$ values mean that the effective value of $\alpha$ is reduced by a factor of $\kappa$
and so the total number of added neurons $N_T$ is now only bounded by $N_T < \kappa \frac{\lambda}{\alpha}$.

We will now try to describe the effect of more general failures of $\mF$ to represent the true curvature $\mH$.
Local expansion behaviour, i.e. without further parameter optimization, is bounded by lemma \ref{lemma:tau_rate_bound} of appendix \ref{app:proofs}.
Assuming the baseline case of $\mH = \mF$, we may substitute $\lambda = 2\Delta L_\text{max}$.
If we assume small step sizes, the rate of loss reduction $\Dot{L} = -\eta$ is given by the natural expansion score by definition,
regardless of $\mH$.
If at all times $t$ during training the rate of reduction of expansion score $-\Dot{\eta}(t) < -\Dot{\eta}_B(t)$ is lower than the baseline scenario,
then $\eta$ will at all times be greater than expected.
Since the rate of loss reduction $\Dot{L} = \eta$ is given by $\eta$, $L$ will decrease faster than expected and the remaining maximum possible loss reduction $\Delta L_\text{max}$ will be at all times less than expected.
It can be seen from lemma \ref{lemma:tau_rate_bound} that discrepancies in these directions relative to baseline will result in fewer additions being made.

We now only need to establish conditions under which the actual rate of reduction in $\eta$ is lower than the expected rate.
The rate of change during optimization (indicated by overdot) of the various components of $\eta$ can be described as follows:
\begin{align}
    \Dot{\vtheta} & = -\mF^{-1} \vg \\
    \Dot{\vg} & = \mH \Dot{\vtheta} = -\mH \mF^{-1} \vg \\
    \Dot{\vg}^T & \mF^{-1} \vg  = -\vg^T \mF^{-1} \mH \mF^{-1} \vg \\
    \Dot{\eta} &  = -\Dot{\vg}^T \mF^{-1} \vg - \vg^T \mF^{-1} \Dot{\vg} - \vg^T \mF^{-1} \Dot{\mF} \mF^{-1} \vg \\
    & = -\vg^T \mF^{-1} \left( 2\mH + \Dot{\mF} \right) \mF^{-1} \vg
\end{align}
Since in the base case $\mH_B = \mF$ and $\Dot{\mF}_B = 0$, we have that
if $\mH + \frac{1}{2} \Dot{\mF} \preceq \mF$
then $-\Dot{\eta} \leq -\Dot{\eta}_B$.
Putting the above results together, we have that if at all times during training $\mH + \frac{1}{2} \Dot{\mF} \preceq \mF$,
then the bound on total additions $N_T < \frac{\lambda}{\alpha}$ should hold.
Incorporating the previous result regarding $\mH = \kappa \mF$,
it also appears that if at all times $\mH + \frac{1}{2}\Dot{\mF} \preceq \kappa \mF$, then $N_T < \kappa \frac{\lambda}{\alpha}$.
Assuming $\mF$ positive definite and the loss surface smooth (i.e. $\mH$ and $\Dot{\mF}$ finite),
then there will exist some finite $\kappa$ for which the condition holds and so $N_T$ will be bounded.





\section{Accounting for Activation Function Curvature}
\label{app:ubah}
The Fisher matrix and its standard variants (empirical fisher, euclidean metric in output space, Fisher metric in output space, Generalized Gauss-Newton (GGN), etc.)
are all of the form $\mJ^T \mH_y \mJ$, where $\mJ$ is the parameters-to-outputs jacobian and $\mH_y$ is some symmetric positive definite matrix in the output space.
If the loss function is convex then the GGN (which uses the exact loss function curvature as $\mH_y$) captures all curvature originating here, but it implicitly linearizes the network itself through use of $\mJ$.
We noted at the end of 3.4 that we sometimes use a modification to the Fisher/GGN which captures more of the total curvature, which we will now describe.
Consider a single layer perceptron $\mathcal{L} \circ \mW_2 \circ \sigma \circ \mW_1: \vx \rightarrow \mathbb{R}$, with output $\vy$, pre-activations $\vp$, and hidden activations $\vh$.
It consists of a convex loss function $\mathcal{L}: \vy \rightarrow \mathbb{R}$, linear functions $\mW_1: \vx \rightarrow \vp$ and $\mW_2: \vh \rightarrow \vy$, and a nonlinearity $\sigma: \vp \rightarrow \vh$ which we assume to be a second differentiable function $\mathbb{R} \rightarrow \mathbb{R}$ vectorized over $\vp$.

If we are interested in approximating the second derivative with respect to $\vx$, the GGN in this case is given by $\mW_2^T \sigma' \mW_1^T \mH_y \mW_1 \sigma' \mW_2$, where $\mH_y = \partial^2 \mathcal{L}/\partial \vy^2$ is the (positive semi-definite) hessian of $\mathcal{L}$ and $\sigma'$ is the (diagonal) derivative of $\sigma$.
If $\sigma$ were linear, i.e. $\sigma'' = 0$, this would be exact, but the full curvature with respect to $\vx$ in general includes more terms.
In particular,
\begin{equation}
\mH_x := \frac{\partial^2 \mathcal{L}}{\partial \vx^2} = \mW_1^T \sigma' \mW_2^T \mH_y \mW_2 \sigma' \mW_1 + \mW_1^T \texttt{diag}(\frac{\partial \mathcal{L}}{\partial \vh} \odot \sigma'') \mW_1.
\label{eq:hess_slp_x}
\end{equation}
where $\texttt{diag}(\frac{\partial \mathcal{L}}{\partial \vh} \odot \sigma'')$ is the diagonal matrix with diagonal entries given by the elementwise product of the gradient with respect to the hidden activations $\vh$ and the second derivative of $\sigma$ for each hidden feature $\sigma''$.

Generalizing the above to an MLP or a CNN, the curvature $\mH_x$ is given by a sum over all evaluations of $\sigma$ during the forward pass for each layer:
\begin{equation}
    \mH_x = \mJ_{yx}^T\mH_y \mJ_{yx} + \sum_l \sum_i g_i \sigma''(p_i) \mJ_{ix}^T \mJ_{ix}
\end{equation}
where $g_i = \frac{\partial \mathcal{L}}{\partial h_i}$ is the (scalar) gradient of the loss with respect to $h_i$, $h_i$ and $p_i$ are the (scalar) post and pre-activations respectively for the $i$th evaluation of $\sigma$ in layer $l$, $\frac{\partial \mathcal{L}}{\partial h_i}$ is the gradient with respect to the post-activation $i$, $\mJ_{ix} \in \mathbb{R}^{1 \times \texttt{dim}(\vx)}$ is the (vector-valued) jacobian of $p_i$ with respect to $\vx$, and $\mJ_{yx}$ is the normal (matrix) jacobian of $\vy$ with respect to $\vx$.
It can be seen that the extra curvature is given by a weighted sum of vector outer products $\mJ_{ix}^T \mJ_{ix}$, one for each evaluation of $\sigma$.
While outer products are always positive semi-definite, $\mH_x$ is not because the product $g_i \sigma''(p_i)$ is not always positive.
We therefore replace each $g_i \sigma''(p_i)$ with the upper bound given by the absolute value $|g_i \sigma''(p_i)|$ and obtain the corresponding upper bound on $\mH_x$:
\begin{equation}
    \mJ_{yx}^T \mH_y \mJ_{yx} + \sum_l \sum_i |g_i \sigma''(p_i)| \mJ_{ix}^T \mJ_{ix} \succeq \mH_x.
    \label{eq:ubah_x}
\end{equation}

We have described an upper bound on the hessian with respect to the input of a NN.
This clearly also gives an upper bound on the hessian with respect to the intermediate activations by truncating the network.
We are actually interested in the curvature with respect to the parameters $\vtheta$, and this may be obtained by substituting $\mJ_{i\theta}$ for $\mJ_{ix}$ and $\mJ_{y\theta}$ for $\mJ_{yx}$ in equation \ref{eq:ubah_x} above, where $\mJ_{i\theta}$ and $\mJ_{y\theta}$ are respectively the jacobians of the $i$th pre-activation and the output $\vy$ with respect to $\vtheta$.
Due to the second derivative of the composite function of two linear layers $\mW_a \circ \mW_b$ with respect to the parameters of $\mW_a$ and $\mW_b$ being non-zero, we have still neglected some terms connecting the different layers, but we use a block-diagonal KFAC approximation anyway, so this does not matter.
We have therefore accounted for the extra curvature terms neglected by the GGN in our new approximation.
In particular, as one goes earlier into the network, there are more and more nonlinear activation functions between the parameters and the loss function,
so the GGN can systematically underestimate the magnitude of the curvature for earlier layers.
In SENN, this would result in excessive capacity addition to earlier layers, the avoidance of which motivates our use of this alternative.

The reader may note that we have very many outer products in our sum (equation \ref{eq:ubah_x}) and worry that this makes our approximation computationally intractable.
Naively for a CNN one would have to compute an outer product of a parameters-sized vector for every channel and pixel of every layer for every data item.
Fortunately, it is possible to use a stochastic estimator for the Fisher which has constant time scaling in the number of network outputs.
We use a similar trick to obtain an estimator for our curvature with a cost only about a factor of two higher.
For full details, see the attached implementation of our experiments, but we will now summarize.

If $\mJ_{y\theta}^T \mH_y \mJ_{y\theta}$ is the GGN and $\mL$ is a cholesky factorization of the positive definite output curvature, $\mL \mL^T = \mH_y$,
then we can construct a stochastic estimator of the GGN via vector jacobian products.
Specifically, if $\xi \sim \mathcal{N}(\mathbf{0}, \mI)$ is a normally distributed vector of dimension equal to $\vy$,
then let $\vv = \mJ_{y\theta}^T \mL \xi$.
It can be seen that $\vv \sim \mathcal{N}(\mathbf{0}, \mJ_{y\theta}^T\mL \mI \mL^T \mJ_{y\theta})$ has covariance equal to the GGN.
The GGN may therefore be approximated by taking $N$ samples of $\vv$ for a total cost of $N$ backpropagations, independently of the dimension of $\vy$ (assuming the factor $\mL$ already given).
In our case we do the same thing, but additionally introduce random gradient covectors $\zeta_i \sim \mathcal{N}(0, \sqrt{|g_i \sigma''(p_i)|})$ during backpropagation at each activation function evaluation $i$ when computing the vector jacobian product.
This depends on the local gradient $g_i$ which must first be computed via backpropagation, but at worst this only results in two backpropagation evaluations,
and in fact can be implemented as a modified hessian vector product in JAX, see the included implementation code.

\section{Rank One Updates}
\label{app:rank_one_updates}
While it is well known that the inverse of a matrix may be efficiently tracked by rank one updates,
we will show in this appendix how the inverse square root may be tracked similarly with only matrix vector products, and how this allows the maintenance of a kronecker factored approximation to the inverse square root of some positive semi-definite curvature approximation $\mF$.
Assume we have a symmetric positive definite matrix $\mA$ and its inverse $\mA^{-1}$.
If we make a rank one update to $\mA$ by adding some outer product $\va \va^T$ then we may cheaply update our stored inverse $\mA^{-1}$ by the Sherman-Morrison formula:
\begin{equation}
    \label{eqn:sherman}
    (\mA + \va \va^T)^{-1} = \mA^{-1} - \frac{\mA^{-1}\va \va^T \mA^{-1}}{1 + \va^T \mA^{-1} \va}
\end{equation}

While direct inversion of the new matrix is $\mathcal{O}(N^3)$, this update is merely $\mathcal{O}(N^2)$.
It is often advantageous to track the square root of a matrix or its inverse square root, for example via the cholesky decomposition of a postitive definite matrix.
Unfortunately, while there does exist an $\mathcal{O}(N^2)$ rank one update to the cholesky factor, this update is not composed of matrix-vector products and is therefore not easy to vectorize on a GPU.
We therefore note the existence of a modification to the Sherman-Morrison formula for a rank one update $\mL'$ such that $\mL' \mL'^T = (\mA + \va \va^T)^{-1}$ to some existing inverse square root $\mL$ with $\mL \mL^T = \mA^{-1}$:
\begin{equation}
    \label{eqn:rank_one_iroot}
    \mL' = \mL +
    \frac{\sqrt{1 - \frac{\mu}{1+\mu}}-1}{\mu}
    \mL \mL^T \va \va^T \mL
\end{equation}
where
$\mu = \va^T \mL \mL^T \va$.
Here, the coefficient can be found as the solution to a quadratic equation after forming the product $\mL' \mL'^T$ and comparing coefficients to the Sherman-Morrison formula.
There are many square roots of $\mA$ and this formula does not find some specific unique square root, unlike the cholesky factorization, but it is composed of only matrix-vector products and therefore vectorizes easily.

From appendix \ref{app:ubah} we have a cheap stochastic estimator producing sample vectors $\vv \sim \mathcal{N}(\mathbf{0}, \mF)$ distributed according to some curvature approximation $\mF$.
Combined with a simple exponential moving average these provide a good estimate of $\mF$.
We then simply apply equation $\ref{eqn:rank_one_iroot}$ to track the inverse root of $\mF$ and therefore also its full inverse.
If one tracks $\mF$ via the obvious means of tracking outer products then one also has the simple square root of $\mF$ in combination with the above.

In order to track a KFAC approximation to $\mF$ one further splits the samples $\vv$ into stochastic estimates of the factors $\mA$ and $\mS$.
For example, let $\mV$ be some sample $\vv$ restricted to the parameters of a single linear transform $\mW \in \mathbb{R}^{M\times N}$ and interpreted as a matrix.
Then $\vs = \frac{1}{N}\mV \xi$ with $\xi \sim \mathcal{N}(\mathbf{0}, \mI_N)$ has outer products with expectation $\mS$, and similarly $\va = \frac{1}{M} \mV^T \zeta$ with $\zeta \sim \mathcal{N}(\mathbf{0}, \mI_M)$ has outer products with expectation $\mA$.

Concretely, given a sample $\vv$ with covariance $\mF$ we split it into chunks for each layer separately, and then do the following for each: first update an estimate for the inverse root of $\mA$, then for the inverse root of $\mS$, and then finally estimate a diagonal variance in the resulting kronecker factored whitened basis.
To ensure all eigenvalues remain finite and reasonably sized we additionally add a unit covariance diagonal element to our curvature estimate $\mF$ by adding appropriately scaled unit normal vectors to each sample $\vv$.

While we have the full factor $\mA$ for all neurons in the network, used or unused, we track the inverse root for the subset of $\mA$ corresponding to the used neurons.
The other neurons are proposals, and when they are activated this subset of $\mA$ increases.
Clearly, under these circumstances, the inverse root must also be updated.
There is a simple and reliable way to do this: compute the inverse cholesky factor of the new subset of $\mA$ and set $\mL$ to this.
Since it is easy to verify whether some $\mL$ is a correct inverse root of $\mA$ by comparing $\mL^T \mA \mL$ to $\mI$,
we in fact use various approximate rank one updates to $\mL$ and fall back on the cholesky factorization only when the error grows large.
Overall, our use of the KFAC approximation is a natural fit for SENN, because the addition or removal of a neuron only touches one of the factors at a time,
and there are many opportunities to make such incremental updates efficient.

\end{document}